\documentclass{article}

\usepackage{microtype}
\usepackage{graphicx}
\usepackage{booktabs} 
\usepackage{hyperref}

\usepackage[accepted]{icml2024}

\usepackage{amsmath}
\usepackage{amssymb}
\usepackage{mathtools}
\usepackage{amsthm}

\usepackage[capitalize,noabbrev]{cleveref}

\usepackage{enumitem}
\usepackage{bm}
\usepackage{subcaption}

\graphicspath{{../}}

\theoremstyle{plain}
\newtheorem{theorem}{Theorem}[section]

\newtheorem{lemma}[theorem]{Lemma}

\theoremstyle{definition}
\newtheorem{definition}[theorem]{Definition}
\newtheorem{assumption}[theorem]{Assumption}
\theoremstyle{remark}
\newtheorem{remark}[theorem]{Remark}

\usepackage[textsize=tiny]{todonotes}

\DeclareMathOperator{\Xb}{\textbf{X}}
\DeclareMathOperator{\xb}{\textbf{x}}
\DeclareMathOperator{\yb}{\textbf{y}}
\DeclareMathOperator{\thetab}{\bm\theta}
\DeclareMathOperator{\Thetab}{\bm\Theta}
\DeclareMathOperator{\Loss}{\mathcal L}
\DeclareMathOperator{\B}{\mathcal B}

\newcommand{\lrVert}[1]{\left\Vert #1 \right\Vert}

\icmltitlerunning{PBM-VFL: Vertical Federated Learning with Feature and Sample Privacy}

\begin{document}

\twocolumn[
\icmltitle{PBM-VFL: Vertical Federated Learning with Feature and Sample Privacy}




\begin{icmlauthorlist}
\icmlauthor{Linh Tran}{rpi}
\icmlauthor{Timothy Castiglia}{rpi}
\icmlauthor{Stacy Patterson}{rpi}
\icmlauthor{Ana Milanova}{rpi}
\end{icmlauthorlist}

\icmlaffiliation{rpi}{Rensselaer Polytechnic Institute}

\icmlcorrespondingauthor{Linh Tran}{tranl3@rpi.edu}

\icmlkeywords{Machine Learning, ICML}

\vskip 0.3in
]



\printAffiliationsAndNotice{} 

\begin{abstract}
We present Poisson Binomial Mechanism Vertical Federated Learning (PBM-VFL), a communication-efficient Vertical Federated Learning algorithm with Differential Privacy guarantees. PBM-VFL combines Secure Multi-Party Computation with the recently introduced Poisson Binomial Mechanism to protect parties' private datasets during model training. We define the novel concept of feature privacy and analyze end-to-end feature and sample privacy of our algorithm. We compare sample privacy loss in VFL with privacy loss in HFL. We also provide the first theoretical characterization of the relationship between privacy budget, convergence error, and communication cost in differentially-private VFL. Finally, we empirically show that our model performs well with high levels of privacy.
\end{abstract}

\section{Introduction}\label{intro.sec}

Federated Learning (FL) \cite{pmlr-v54-mcmahan17a} is a machine learning technique where data is distributed across multiple parties,
and the goal is to train a global model collaboratively.
The parties execute the training algorithm, facilitated by a central server, 
without directly sharing the private data with each other or the server. 
FL has been used in various applications such as drug discovery, mobile keyboard prediction, and ranking browser history suggestion \cite{9153560}.

The majority of FL algorithms support Horizontal Federated Learning (HFL) (e.g., \cite{DBLP:journals/tist/YangLCT19}), where the datasets of the parties are distributed horizontally, i.e, all parties share the same features, but each has a different set of data samples. In contrast, Vertical Federated Learning (VFL) targets the case where all parties share the same set of data sample IDs, but each has a different set of features (e.g., ~\cite{FDML, gu2021privacy}). An example of VFL includes a bank, a hospital, and an insurance company who wish to train a model predicting customer credit scores. The three institutions have a common set of customers, but the bank has information about customers' transactions, the hospital has medical records, and the insurance company has customers' accident reports. Such a scenario must employ VFL to train models on private vertically distributed data.

While FL is designed to address privacy concerns by keeping data decentralized, there is possible information leakage from the messages exchanged during training \cite{DBLP:conf/nips/GeipingBD020}, \cite{MahendranV15}. So, it is crucial to develop methods for FL that have provable privacy guarantees. A common approach for privacy-preservation is \emph{Differential Privacy} (DP)~\cite{10.1561/0400000042}, in which the private data is protected by adding noise at various stages in the training algorithm. Through careful application of DP, one can protect the data, not only in a single computation, but throughout the execution of the training algorithm.
A number of works have developed and analyzed the end-to-end privacy of DP-based  HFL algorithms (e.g., ~\cite{10.1145/3338501.3357370, 9069945, 10.1145/3378679.3394533}). However, there is limited prior work on privacy analysis in VFL, and none that addresses the interplay between the convergence of the training algorithm, the end-to-end privacy, and the communication cost.

We propose Poisson Binomial Mechanism VFL (PBM-VFL), a new VFL algorithm that combines the Poisson Binomial Mechanism (PBM) \cite{chen2022poisson} with Secure Multi-Party Computation (MPC) to provide DP over the training datasets. In PBM-VFL, each party trains a local network that transforms their raw features into embeddings. The server trains a fusion model that produces the predicted label from these embeddings. To protect data privacy, the parties quantize their embeddings into differentially private integer vectors using PBM. The parties  apply MPC over the integer values so that the server aggregates the quantized sum without learning anything else. The server then estimates the embedding sum to calculate the loss and the gradients needed for training. 

To study the privacy of PBM-VFL, we introduce the novel notion of \emph{feature privacy} which aims to protect the full (column-based) feature data held by a party, across all samples. This notion arises naturally in VFL, but requires a new definition of adjacent data sets for DP. We consider the standard \emph{sample privacy} as well, specifically privacy loss per sample arising from the nature of computation in VFL.
We note that VFL presents a different privacy problem than HFL. In HFL, parties share an aggregate gradient over a minibatch  with the server, whereas in VFL, the server learns the sum of the party embeddings \emph{for each sample} in a minibatch individually. This difference requires a different privacy analysis. Further, the DP noise enters via the gradient computation in HFL, whereas it enters via an argument to the loss function in VFL. This change necessitates different convergence analysis. 

We analyze the end-to-end feature and sample privacy of PBM-VFL as well as its  convergence behavior. We also relate this analysis to the communication cost. Through these analyses, we provide the first theoretical characterization of the tradeoffs between privacy, convergence error, and communication cost in VFL.

We summarize our main contributions in this work.
\begin{enumerate}[nosep,topsep=0pt]
    \item We introduce PBM-VFL, a communication efficient and private VFL algorithm.
    \item We provide analysis of the overall privacy budget for both feature privacy and sample privacy.
    \item We prove the convergence bounds of PBM-VFL and give the relationship between the privacy budget and communication cost.  
    \item We evaluate our algorithm by training on ModelNet-10 and Cifar-10 datasets. Our results show that the VFL model performs well with high accuracy as we increase the privacy parameters.
\end{enumerate}

\textbf{Related work.}
The problem of privacy preserving in HFL with DP has received much attention with many significant publications, including \cite{9069945, 10.1145/3378679.3394533,agarwal2018cpsgd}. More recent works propose notable differentially private compression mechanisms for HFL. \cite{chen2022poisson} utilized quantization and DP to protect the gradient computation in an unbiased manner. \cite{guo2023privacyaware} proposed interpolated Minimum Variance Unbiased quantization scheme with R\'{e}nyi DP to protect client data. \cite{youn2023randomized} developed Randomized Quantization Mechanism to achieve R\'{e}nyi DP and prevent data leakage during local updates.
All these works focus on quantization methods with DP in HFL. While PBM-VFL utilizes the same mechanisms, there are significant differences for application and analysis in VFL vs HFL.

There are previous papers that provide different private methods for VFL. \cite{10.5555/3618408.3619245, 9343209} propose MPC protocols for VFL, but they do not use DP, and thus while aggregation computations are private, information may still be leaked from the resulting aggregate information.
\cite{10.14778/3583140.3583146, ijcai2022p272} use DP mechanisms for $k$-means and graph neural networks, but their algorithms do not apply to traditional neural networks. \cite{xu2021fedv} proposes a secure and communication-efficient framework for VFL using Functional Encryption, but they do not consider privacy for end-to-end training.

\textbf{Outline.} The rest of the paper is organized as follows. Sec. 2 presents the building blocks for PBM-VFL. Sec. 3 details the system model and training problem. Sec. 4 lists our privacy goals. Sec. 5 presents our algorithm, and Sec. 6 presents the analysis. Sec. 7 summarizes our experimental results, and Sec. 8 concludes.

\section{Background}\label{bg.sec}

This section presents background on Differential Privacy and related building blocks used in our algorithm.

\subsection{Differential Privacy} 

Differential Privacy (DP) provides a strong privacy guarantee that ensures that an individual's sensitive information, e.g., the training data, remains private even if the adversary has access to auxiliary information.
In this work, we employ R\'{e}nyi Differential Privacy (RDP) \cite{Mironov_2017}, which is based on the concept of the R\'{e}nyi divergence. We utilize RDP rather than standard $(\epsilon, \delta)$-DP because it facilitates the calculation of the cumulative privacy loss over the sequence of algorithm training iterations. The R\'{e}nyi divergence and RDP are defined as follows.
\begin{definition}
For two probability distributions $\mathcal{P}$ and $\mathcal{Q}$ defined over a set $\mathcal{R}$, the \emph{R\'{e}nyi divergence of order~$\alpha > 1$} is
\begin{align*}
    D_\alpha (\mathcal{P},\mathcal{Q}) \coloneqq \frac{1}{\alpha - 1} \text{log} \left(\mathbb{E}_{x\sim \mathcal{Q}} \Bigg( \frac{\mathcal{P}(x)}{\mathcal{Q}(x)} \Bigg)^\alpha \right).
\end{align*}
\end{definition}
\begin{definition}
A randomized mechanism ${\mathcal{M}: \mathcal{D} \rightarrow \mathcal{R}}$ with domain $\mathcal{D}$ and range $\mathcal{R}$ satisfies \emph{$(\alpha, \epsilon)$-RDP} if for any two adjacent inputs $d, d' \in \mathcal{D}$, it holds that
\begin{align*}
    D_\alpha (\mathcal{P}_{\mathcal{M}(d)}, \mathcal{P}_{\mathcal{M}(d')}) \leq \epsilon .
\end{align*}
\end{definition}

\subsection{Poisson Binomial Mechanism}
A key component of our algorithm is to protect the inputs of distributed sum computations. To do this, we rely on the combination of RDP and MPC and use the Poisson Binomial Mechanism (PBM) developed in~\cite{chen2022poisson}. Unlike other private methods, PBM provides RDP guarantee, and uses the Binomial distribution to generate scalar quantized values which is suitable for integer-based MPC.

\begin{algorithm}[t]
    \caption{Scalar Poisson Binomial Mechanism}
    \label{pbm.alg}
    \begin{algorithmic}[1]
    \STATE {\textbf{Input:}} $x_i \in [-C, C], \beta \in [0, \frac{1}{4}], b\in \mathbb{N}$
    \STATE $p_i \leftarrow \frac{1}{2} + \frac{\beta}{C} x_i$
    \STATE $q_i \leftarrow \text{Binom} (b, p_i)$
    \STATE {\textbf{Output:}} Quantized value $q_i$
\end{algorithmic}
\end{algorithm}

We sketch the process for computing a sum of scalar values. 
Suppose each participant $m=1, \ldots, M$ has an input ${x_m \in [-C, C]}$, and the goal is to estimate the sum ${s = \sum_{m=1}^M x_m}$ while protecting the privacy of the inputs. Each participant first quantizes its value according to Algorithm~\ref{pbm.alg}. The values of $\beta\in [0, \frac{1}{4}]$ and $b\in \mathbb{N}$ are chosen to achieve a desired RDP and accuracy tradeoff.
The participants use MPC to find the sum $\hat{q} = \sum_{m=1}^M q_m$.
An estimated value of $s$ is then computed from $\hat{q}$ as ${\tilde{s} =  \frac{C}{\beta b} (\hat{q} - \frac{bM}{2})}$.
We provide the following theorem, which is a slight modification from the result in ~\cite{chen2022poisson} adapted for computing a sum rather than an average. 
\begin{theorem}[\cite{chen2022poisson}]
    \label{pbm.thm}
    Let $x_m \in [-C, C]$, $m=1, \ldots, M$,  $\beta \in [0, \frac{1}{4}]$, and $b\in \mathbb{N}$. Then, the sum computation:
    \begin{enumerate}[nolistsep]
        \item satisfies $(\alpha, \epsilon(\alpha))$-RDP for $\alpha > 1$ and ${\epsilon(\alpha) = \Omega (b\beta^2 \alpha / M)}$
        \item yields an unbiased estimate of $s$ with variance~$\displaystyle \frac{C^2 M}{4 \beta^2 b}$.
    \end{enumerate}
\end{theorem}

\subsection{Multi-Party Computation} \label{secagg.sec}
While the PBM can be used to provide RDP for a sum, we must ensure that the inputs to the sum are not leaked during the computation. This is achieved via MPC, a cryptographic mechanism that allows a set of parties to compute a function over their secret inputs, so that only the function output is revealed.

There are a variety of MPC methods that can be used for sum computation. For the purposes of our communication cost analysis, we fix an MPC protocol, specifically Protocol 0 for Secure Aggregation from~\cite{bonawitz2016practical}. It considers $M$ parties, each holding an integer secret value $q_m \in [0,b)$ and an honest-but-curious server. The goal is to compute $\sum q_m$ collaboratively so that the server learns the sum and nothing else, while the parties learn nothing.

In Protocol 0, each pair of parties $m_1,m_2$ samples two random integers in $[0,b)$, $u_{m_1,m_2}$ and $u_{m_2,m_1}$ using  pseudo-random number generators with a seed known to only parties $m_1$ and $m_2$. Thus, the pseudo-random number generators generate the same integers at party $m_1$ and at party $m_2$ and no exchange is required. Each party then computes $M-1$ perturbations $p_{m,m'} = u_{m,m'}-u_{m',m}$ and masks its secret value by computing $y_m = q_m + \sum^M_{m'=1} p_{m,m'}$ ($p_{m,m} = 0$). It then sends $y_m$ to the server. The server sums all $y_m$: $S = \sum^M_{m=1} y_m + \sum^M_{m=1} \sum^M_{m'=1} p_{m,m'}$, and this is exactly $\sum^M_{m=1} q_m$. At the same time, the protocol is perfectly secure, revealing nothing about the individual $q_m$'s to the server. 

To analyze communication cost, observe that each party incurs cost by sending $y_m$. Since ${-(M-1)b < y_m < Mb}$, $O(\log (bM))$ bits suffice to represent the range of negative and positive values of $y_m$. Thus, we need $O(\log (bM))$ per-party bits to send the masked value.

\section{Training Problem}\label{model.sec}

We consider a system consisting of $M$ parties and a server.
There is a dataset $\Xb \in \mathbb{R}^{N \times D}$ partitioned 
across the $M$ parties, where $N$ is the number of data samples
and $D$ is the number of features. 
Let $\xb^i$ denote the $i$th sample of $\Xb$. 
For each sample $\xb^i$, each party $m$ holds a disjoint subset, i.e., a vertical partition, of the features.
We denote this subset by $\xb_m^i$, and note that $\xb^i = [\xb_1^i, \ldots, \xb_M^i]$.
The entire vertical partition of $\Xb$ that is held  by party $m$ is denoted by $\Xb_m$, with $\Xb = [\Xb_1, \ldots, \Xb_M]$.

Let $y_i$ be the label for sample $\xb_i$, and let $\yb \in \mathbb{R}^{N \times 1}$ denote the set of all  labels. We assume that the labels are stored at the server.
As it is standard, the dataset is aligned for all parties in a privacy-preserving manner as a pre-processing step. This can be done using Private Entity Resolution~\cite{xu2021fedv}.

\begin{figure}
\centering
\includegraphics[scale=.05]{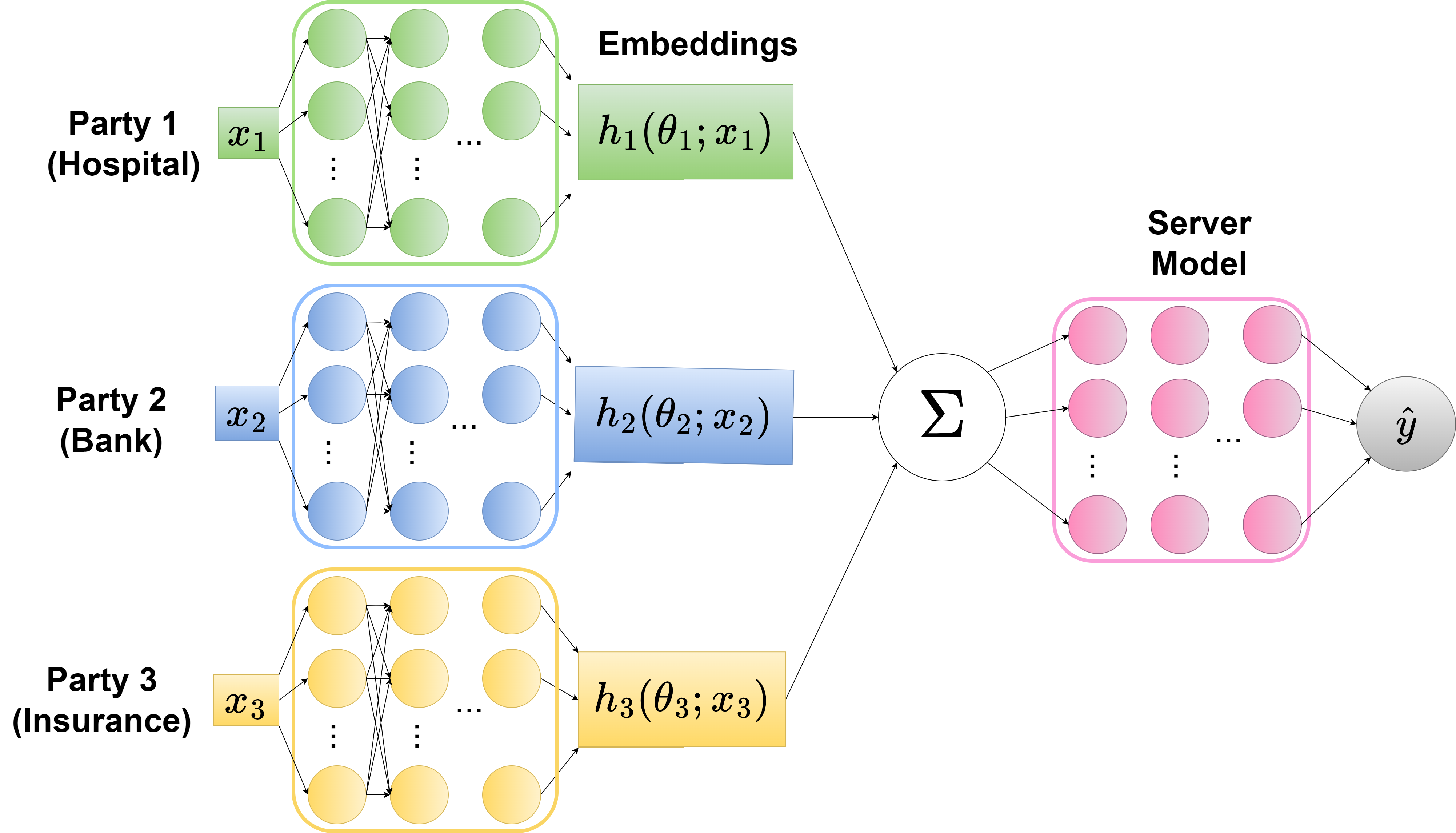}
\caption{Example global model with neural networks.}
\label{model.fig}
\end{figure}

The goal is to train a global model using the data from all parties and the 
labels from the server. Each party $m$ trains a local network $h_m$ with 
parameters $\thetab_m$ that takes the vertical partition of a sample 
$\xb$ as 
input and produces an embedding of dimension $P$. The server trains a fusion 
model $h_0$ with parameters $\thetab_0$ that takes a sum of the embeddings 
for a sample as input and produces a predicted label $\hat{y}$. 
The global model $f(\cdot)$ has the form
\begin{align}
f(\xb; \Thetab) := h_0\left(\sum_{m=1}^M h_m(\thetab_m; \xb_m);~\thetab_0\right)
\end{align}
where $\Thetab$ denotes the set of all model parameters. An example of the global model architecture is shown in Figure~\ref{model.fig}. To train this model, the parties and the server collaborate to minimize a loss function: 
\begin{align}
\Loss(\Thetab; \Xb, \yb) := \frac{1}{N}\sum_{i=1}^N \ell\left(\thetab_0, \hat{h}(\thetab_1, \ldots, \thetab_M; \xb^i); y^i \right)
\end{align}
where $\ell(\cdot)$ is the loss for a single sample $(\xb^i, y^i)$, 
and 
\begin{align}
\hat{h}(\thetab_1, \ldots, \thetab_M; \xb^i) = \sum_{m=1}^M h_m(\thetab_m;\xb_m^i).
\end{align}
Let $\B \coloneqq (\Xb^{\B}, \yb^{\B})$ be a randomly sampled mini-batch of $B$ samples.
We denote the partial derivative of $\Loss$ over $\B$ with respect to $\thetab_m$, $m=0, \ldots, M$,  by
\begin{align*}
&\nabla_m \Loss_{\B}(\Thetab) \coloneqq \nonumber \\
&~~~~~~\frac{1}{B} \sum_{(\xb^i,y^i) \in \B} 
\nabla_{\thetab_m} \ell(\thetab_0, \hat{h}(\thetab_1, \ldots, \thetab_M; \xb^i); y^i). 
\end{align*}

The partial derivatives of $\Loss$ and $\Loss_{\mathcal{B}}$ with respect to $\hat{h}$ are denoted by $\nabla_{\hat{h}} \Loss$ and $\nabla_{\hat{h}} \Loss_{\mathcal{B}}$, respectively.

We make the following assumptions. 
\begin{assumption}(Smoothness). \label{smooth.assum} 
    \begin{enumerate}[noitemsep]
    \item There exists positive constant ${L < \infty}$ 
    such that for all $\Thetab_1$, $\Thetab_2$: 
        \begin{align}
            \lrVert{\nabla \Loss(\Thetab_1) - \nabla \Loss(\Thetab_2)} &\leq L \lrVert{\Thetab_1 - \Thetab_2} \label{smooth1.eq}
        \end{align}
\item There exist positive constants $L_0 < \infty$ and ${L_{\hat{h}} < \infty}$ such that for all server parameters $\thetab_0$ and $\thetab_0'$ and 
all embedding sums $\bm{h}$ and $\bm{h}'$:
\begin{align}
& \| \nabla_{0} \ell(\thetab_0,\bm{h}) - \nabla_{0}\ell(\thetab_0',\bm{h}') \| \leq \nonumber \\
&~~~~~~~~~~~~~~~~~~~~~~~~~~~~ L_0 \|[ \thetab_0^T, \bm{h}^T]^T - [ \thetab_0'^T, \bm{h}'^T]^T \| \label{smooth2.eq} \\
& \| \nabla_{\hat{h}} \ell(\thetab_0,\bm{h}) - \nabla_{\hat{h}}\ell(\thetab_0',\bm{h}')  \| \leq \nonumber \\
&~~~~~~~~~~~~~~~~~~~~~~~~~~~  L_{\hat{h}} \| [ \thetab_0^T, \bm{h}^T]^T - [ \thetab_0'^T, \bm{h}'^T]^T \|. \label{smooth3.eq}
\end{align}
\end{enumerate}
\end{assumption}
\begin{assumption}[Unbiased gradients]
    \label{bias.assum}
    For every mini-batch $\B$, 
        the stochastic gradient is unbiased:
        \begin{align*}
            \mathbb{E}_{\B}{\nabla \Loss_{\B}(\Thetab)} = \nabla \Loss(\Thetab).
        \end{align*}
\end{assumption}
\begin{assumption}[Bounded variance]
    \label{var.assum} 
    There exists positive constant ${\sigma < \infty}$
        such that for every mini-batch $\B$ (with $| \B | = B$)
        \begin{align}
            \mathbb{E}_{\B}{\lrVert{\nabla \Loss(\Thetab) -  \nabla \Loss_{\B}(\Thetab)}^2} \leq \frac{\sigma^2}{B}.
        \end{align}
\end{assumption}
\begin{assumption}[Bounded embeddings] \label{embedbound.assum} 
There exists positive constant $C < \infty$ such that for $m=1, \ldots, M$, for all $\thetab_m$ and $\xb_m$, $\| h_m(\thetab_m; \xb_m) \|_{\infty} \leq C$.
\end{assumption}
\begin{assumption}[Bounded embedding gradients] \label{embedgrad.assum}
There exists positive constants $H_m < \infty$ for 
$m=1, \ldots, M$ such that for all $\thetab_m$ and all samples $i$, the embedding gradients are bounded as 
\begin{align}
\| \nabla_m h(\thetab_m; \xb_m^{i}) \|_{\mathcal{F}} \leq H_m
\end{align}
where $\| \cdot \|_{\mathcal{F}}$ denotes the Frobenius norm.
\end{assumption}

Part 1 of Assumption~\ref{smooth.assum},  Assumption~\ref{bias.assum}, and Assumption~\ref{var.assum} are standard in the analysis of gradient-based algorithms (e.g., \cite{HogWild18, bottou2018optimization}).
Part 2 of Assumption~\ref{smooth.assum} bounds the rate of change of the partial derivative of $\ell$ with respect to each of its two arguments. This assumption is needed to ensure convergence over the noisy embedding sums. We note that this assumption does not place any additional restrictions on the server model architecture over Assumption~\ref{smooth.assum} Part 1.
Assumption~\ref{embedbound.assum} bounds the individual components of the embeddings. This can be achieved via a standard activation function such as 
sigmoid or $tanh$. Assumption~\ref{embedgrad.assum} bounds the partial derivatives of the embeddings with respect to a single sample. This bound is also necessary to analyze the impact of the DP noise on the algorithm convergence. 



\section{Privacy Goals}

Our method makes use of RDP that aims to provide a measure of indistinguishability between adjacent datasets consisting of multiple data samples. We consider two notions of privacy, the novel \emph{feature privacy} and the standard \emph{sample privacy}. 

Feature privacy is a natural goal in VFL: parties share the same set of sample IDs but each has different feature set, and each party aims to protect its feature set. A natural question is, how much one can learn about a party's data (made up of columns of $\Xb$) from the aggregates of embeddings shared during training?
To this end, we aim to provide feature privacy for each party's feature set by providing indistinguishability of whether a feature was used in training or not. 
For feature privacy, we say that two datasets are \emph{feature set adjacent} if they differ by a single party's feature set. Given an algorithm $\mathcal{M}$ that satisfies RDP, a dataset $d$ and its feature set adjacent dataset $d'$, then RDP ensures that it is impossible to tell, up to a certain probabilistic guarantee, if $d$ or $d'$ is used to compute $\mathcal{M}(d)$. 

This notion of privacy is different from the one in HFL where a party aims to protect a sample, meaning that one cannot tell if a particular sample is present in the training data. 
The standard definition of adjacent datasets applies here: two sets are \emph{sample adjacent} if they differ in a single sample. Standard \emph{sample privacy} is still a concern in VFL and we aim to protect sample privacy as well.

We assume that all parties and the server are honest-but-curious. They correctly follow the training algorithm, but they can try to infer the data of other parties from information exchanged in the algorithm. We assume that the parties do not collude and that communication occurs through robust and secure channels.

\section{Algorithm} \label{algorithm.sec}

\begin{algorithm}[t]
    \caption{PBM-VFL}
    \label{vflmpc.alg}
    \begin{algorithmic}[1]
    \STATE {\textbf{Initialize:}} $\Thetab^0 = [\thetab_0^0,\thetab_1^0, \ldots, \thetab_M^0]$ 
  \FOR {$t \leftarrow 0, \ldots, T-1$}
        \STATE Randomly sample $\B^t$ from $(\Xb, \yb)$
        \FOR {party $m \leftarrow 1, \ldots, M$ in parallel}
        \STATE /* Generate quantized embeddings for $\B^t$ */
            \STATE  $q_m^t \leftarrow \textbf{PBM}(h_m(\thetab_m^t ; \Xb^{\B^t}_m), b, \beta)$ 
        \ENDFOR
        \STATE /* At server */
        \STATE $\hat{q}^t \gets \sum_{m=1}^M q_m^t$ via \textbf{MPC}
        \STATE  $\tilde{h}^t \leftarrow \frac{1}{\beta b} (\hat{q}^t - \frac{bM}{2})$
        \STATE Server sends $\nabla_{\hat{h}} \Loss_{\B}(\thetab_0^t, \tilde{h}^t)$ to all parties
         \STATE /* server updates its parameters */
        \STATE $\thetab_0^{t+1} \gets \thetab_0^t - \eta \nabla_0 \Loss_{\B}(\thetab_0^t, \tilde{h}^{t})$
        \FOR {$m \leftarrow 1, \ldots, M$ in parallel}
            \STATE /* party $m$ updates its parameters */
            \STATE $\nabla_m \Loss_{\B}(\Thetab^t) \gets$
            \STATE $~~~~~~\nabla_m h_m(\thetab_m^t; \Xb_m^t) \nabla_{h_m} \tilde{h}^{t}  \nabla_{\hat{h}} \Loss_{\B^t}(\thetab_0^t, \tilde{h}^t)$
            \STATE $\thetab_m^{t+1} \gets \thetab_m^t - \eta^t \nabla_m \Loss_{\B}(\Thetab^t)$
        \ENDFOR
\ENDFOR
\end{algorithmic}
\end{algorithm}

We now present PBM-VFL. Pseudocode is given in Algorithm~\ref{vflmpc.alg}. 

Each party $m$ and the server initialize their local parameters~$\thetab_m$, $m=0, \ldots, M$ (line 1).
The algorithm runs for $T$ iterations.
In each iteration $t$, the server and parties agree on a minibatch $\B^t$, chosen at random from $\Xb$. This can be achieved, for example, by having the server and all parties use pseudo-random number generators initialized with the same seed.
Each party $m$ generates an embedding $h_m(\thetab_m^t ; \xb^i)$ for each sample $i$ in the minibatch. We denote the set of party $m$'s embeddings for the minibatch by $h_m(\thetab_m^t ; \Xb_m^{\B^t})$.  
Each party $m$ computes the set of noisy quantized embeddings $q_m^t$ using PBM component-wise on each embedding (lines 3-7).

To complete forward propagation, the server needs an estimate of the embeddings sum for each sample in $\B^t$. The parties and the server execute MPC (Protocol 0), which reveals $\hat{q}^i$ for each sample $i \in B^t$ to the server. 
For each $i \in \B^t$ the server estimates the embedding sum as $\tilde{h}^i = \frac{C}{\beta b} (\hat{q}^i - \frac{bM}{2})$ (lines 9-10).
We let $\tilde{h}^t$ denote the set of noisy embedding sums.

The server calculates the gradient 
of $\Loss_{\B}$ with respect to $\hat{h}$ for the minibatch using $\tilde{h}^t$, denoted $\nabla_{\hat{h}} \Loss_{\B}(\thetab_0^t , \tilde{h}^{t})$ and sends this information to the parties for local parameter updates (line 11).
Then the server calculates the stochastic gradient of $\Loss$ with respect to its own parameters, denoted $\nabla_0 \Loss_{\B}(\thetab_0^t , \tilde{h}^{t})$, and uses this gradient to update its own model parameters with learning rate $\eta$ (line 13).
Finally, each party uses the partial derivative received from the server to compute the partial derivative of $\Loss_{\B^t}$ with respect to its local model parameters using the chain rule as:
\begin{align*}
    \nabla_m \Loss_{\B^t}(\Thetab^t) =  \nabla_m h_m(\thetab_m^t; \Xb_m^{\B^t}) \nabla_{h_m} \tilde{h}^{t}  \nabla_{\hat{h}} \Loss_{\B^t}(\thetab_0^t, \tilde{h}^t). 
\end{align*}
Note that $\nabla_{h_m} \tilde{h}^{t}$ is the identity operator.
The party then updates its local parameters (lines 16-18) using this partial derivative, with learning rate $\eta$.

\paragraph{Information Sharing.}
There are two places in Algorithm~\ref{vflmpc.alg} where information about $\Xb$ is shared. The first is when the server learns the sum of the embeddings for each sample in a minibatch (line 9). We protect the inputs to this computation via PBM and MPC. The second is when the server sends $\nabla_{\hat{h}} \Loss_{\B}(\thetab_0^t , \tilde{h}^{t})$ to each party. By the post-processing property of DP, this gradient retains the same privacy protection as the sum computation. We give a formal analysis of the algorithm privacy in Section~\ref{analysis.sec}.

\paragraph{Communication Cost.} We now discuss the communication cost of Algorithm~\ref{vflmpc.alg}. Each party sends its masked quantized embedding at a cost of $O(P \log(bM))$ bits, as detailed in~\ref{bg.sec} and the cumulative cost for $M$ parties and mini-batch of size $B$ becomes $O(BMP\log(bM))$.
In the back propagation, the server sends the partial derivatives without quantization to each party, which is the most costly message exchanging step. 
Nevertheless, we save a significant number of bits when the parties send their masked quantized embedding to the server. 
Let $F$ be the number of bits to represent a floating point number. Then the cost of sending these partial derivatives to $M$ parties is $O(B M P F)$. The total communication cost for Algorithm~\ref{vflmpc.alg} is $O(TBMP ( \log(bM) + F))$.

\section{Analysis}\label{analysis.sec}

We now present our theoretical results with respect to the privacy and convergence of PBM-VFL, and we provide a discussion of the tradeoffs between them. 

\subsection{Privacy}

In this section, we analyze the privacy budget of Algorithm~\ref{vflmpc.alg} with respect to both feature and sample privacy. 

\subsubsection{Feature Privacy}
We first give an accounting of the privacy budget across $T$ iterations of our algorithm. 

\begin{theorem}  \label{privacybudget.thm} 
Assume ${\beta \in [0,\frac{1}{4}]}$ and $b \in \mathbb{N}$. Algorithm~\ref{vflmpc.alg} after $T$ iterations satisfies $(\alpha,\epsilon^\mathit{feat}_\mathit{final}(\alpha))$-RDP for feature privacy for $\alpha > 1$ and $\epsilon^{\mathit{feat}}_\mathit{final}(\alpha) = C_0\frac{TBPb\beta^2\alpha}{MN}$ where $C_0$ is a universal constant.
\end{theorem}
\begin{proof}
Consider the universe $\mathcal{D}_M$ of embeddings $\langle h_1,h_2,...h_M\rangle$, where two members of $\mathcal{D}_M$ are adjacent if they differ in a single embedding, e.g., $h_1$ and $h'_1$. 
In other words, these embeddings were produced from samples in feature set adjacent datasets. 
Theorem~\ref{pbm.thm} gives the following privacy guarantee when revealing a noisy sum of embeddings:
  \begin{align}
  & D_\alpha (\mathcal{P}_{h_1+...+h_M},\mathcal{P}_{h'_1+...+h_M}) \le \nonumber \\
  & ~~~~~~~~~~~~~~~~~~~~~~~~~~~~~~~ \epsilon^{\mathit{feat}}(\alpha) = C_0\frac{Pb\beta^2\alpha}{M}. \label{diverge.eq}
  \end{align}
The factor of $P$ accounts for the dimension of an embedding.
More formally, given members $d_M,d'_M \in \mathcal{D}_M$ produced from feature set adjacent datasets, we have 
\[
D_\alpha (\mathcal{P}_{\mathit{PBM}(d_M)},\mathcal{P}_{\mathit{PBM}(d'_M)}) \le \epsilon^{\mathit{feat}}(\alpha) = C_0\frac{Pb\beta^2\alpha}{M}
\]
where $\mathit{PBM}(\cdot)$ denotes embedding sum computed using PBM.
We protect an individual embedding when revealing the noisy sum of embeddings via PBM, and the privacy loss $\epsilon^{\mathit{feat}}(\alpha)$ we incur is $C_0\frac{Pb\beta^2\alpha}{M}$.

By Theorem~\ref{pbm.thm}, for a sample $i$, the computation of the sum $\tilde{h}_i^t$ is ($\alpha,\epsilon^{\mathit{feat}}(\alpha)$)-RDP for any $\alpha>1$ and 
$\epsilon^\mathit{feat}(\alpha) = C_0\frac{Pb\beta^2\alpha}{M}$ (as shown in Equation~\ref{diverge.eq}).
To compute $\nabla_{\hat{h}} \Loss_{\B}(\thetab_0^t , \tilde{h}^{t})$, the server applies a deterministic function on $\tilde{h}$. By standard post-processing arguments, the computation provides ($\alpha,\epsilon^\mathit{feat}(\alpha)$)-RDP with the same $\epsilon^\mathit{feat}(\alpha)$.

Next, we extend the above (per sample $i$) guarantee to the full feature data for each party. 
In each training iteration, for each sample in the minibtach, the sum mechanism runs separately, and all embeddings are disjoint. Therefore, we can apply Parallel Composition~\cite{10.1561/0400000042} to account for the privacy of each party's full set of embeddings, which is the party's feature data.
In one iteration, each party's feature data is protected with a guaranteed privacy budget of $C_0\frac{Pb\beta^2\alpha}{M}$.

At each iteration, the algorithm processes a sample at a rate $B/N$, leading to an expected $TB/N$ number of times that each sample is used in training over $T$ iterations. Accounting for privacy loss across all $T$ iterations leads to $ C_0\frac{TBPb\beta^2\alpha}{MN}$, that is, the algorithm protects the full feature data of a party by $\epsilon^\mathit{feat}_{\mathit{final}}(\alpha) = C_0\frac{TBPb\beta^2\alpha}{MN}$. 
\end{proof}

We remark that privacy amplification does not directly apply in VFL as it does in HFL. This is because the parties and the server know exactly which sample is used in each iteration. 
As a result, we apply standard composition.
Since each sample is expected to occur $\frac{TB}{N}$ times, standard composition yields the above result $C_0\frac{TBPb\beta^2\alpha}{MN}$. 


\subsubsection{Sample Privacy}

 We note that sample privacy loss in VFL is larger compared to privacy loss in comparable DP-based privacy-preserving HFL algorithms. This is due to the nature of computation. 
 VFL computation requires revealing the (noisy) sum of embeddings \emph{for each sample}; this is in contrast to HFL, which reveals a noisy aggregate over a mini-batch.

As seen in Equation \ref{diverge.eq} the privacy of a single embedding is protected by $\epsilon^{\mathit{feat}}(\alpha)$. Therefore, one can ask, what is the privacy loss for the entire sample resulting from revealing the noisy sum $h_1+...+h_M$, and how does that privacy loss compare to the privacy loss incurred in comparable HFL? Put another way, suppose we run PBM-VFL and it has a feature privacy budget of $\epsilon^{\mathit{feat}}(\alpha)$; we are interested in computing the resulting per-sample privacy. Formally, we want to bound the divergence
 $D_\alpha (\mathcal{P}_{h_1+...+h_M},\mathcal{P}_{h'_1+...+h'_M})$,
  given the VFL $\epsilon^{\mathit{feat}}(\alpha)$ bound of $C_0\frac{Pb\beta^2\alpha}{M}$ from Equation~\ref{diverge.eq}. (Recall that $\mathcal{P}_{h_1+...+h_M}$ stands for $\mathcal{P}_{\mathit{PBM}(\langle h_1,...,h_M\rangle)}$.)

We prove the following theorem: 

\begin{theorem} \label{privacy.thm}
 Let $\epsilon^{\mathit{feat}}(\alpha) = C_0\frac{Pb\beta^2\alpha}{M}$. Then we have $D_\alpha(\mathcal{P}_{h_1+...+h_M}, \mathcal{P}_{h'_1+...+h'_M}) \le C_0\frac{Pb\beta^2S_M(\alpha)}{M}$ where $S_M(\alpha) = \bigg((2^{M+1}-2^{M-1}-2)\alpha-(3\cdot2^{M-1}-3M)+\frac{2^{M-1}-1}{2^{M-2}(\alpha-1)}\bigg)$ 
 for every $M \ge 2$.
\end{theorem}

We establish the theorem by making use of our specific aggregation function (summation) and known $\epsilon^{\mathit{feat}}(\alpha) = C_0\frac{Pb\beta^2\alpha}{M}$. The advantage of this approach over using Mironov's general RDP group privacy result (Proposition 2 in~\cite{Mironov_2017}), is that it allows us to obtain a tighter bound and it imposes no restriction on $\alpha$ (other than the standard restriction $\alpha > 1$).

 We demonstrate the case for $M=2$ below and present the full proof in the appendix.
 The proof is by induction on the terms in the sum and follows the intuition. For $M=2$, let $\mathcal{P} = \mathcal{P}_{h_1+h_2}, \mathcal{Q} = \mathcal{P}_{h'_1+h'_2}$, and $\mathcal{R} = \mathcal{P}_{h'_1+h_2}$. 
 
Mironov~\cite{Mironov_2017} states the following inequality (Corollary 4) for arbitrary distributions $\mathcal{P},\mathcal{Q}$ and $\mathcal{R}$ with common support: 
 \[
 D_\alpha(\mathcal{P},\mathcal{Q}) \le \frac{\alpha - 1/2}{\alpha - 1}D_{2\alpha}(\mathcal{P},\mathcal{R}) + D_{2\alpha-1}(\mathcal{R},\mathcal{Q}).
 \]
 
 Using $\epsilon^\mathit{feat}(\alpha)$ and plugging into Corollary 4 above gives
 \[
 D_\alpha(\mathcal{P},\mathcal{Q}) \le \frac{\alpha - 1/2}{\alpha - 1}C_0\frac{Pb\beta^2 2\alpha}{M} + C_0\frac{Pb\beta^2(2\alpha-1)}{M}.
 \]
 Simplification yields the expected term:
 \[
 D_\alpha(\mathcal{P},\mathcal{Q}) \le C_0\frac{Pb\beta^2\bigg(2^2\alpha + \frac{1}{\alpha-1}\bigg)}{M}.
 \]

\begin{remark}[Comparison with HFL]
 One can easily show that $S_M(\alpha) > M\alpha$ holds for $M \ge 2$. Thus, $C_0\frac{Pb\beta^2S_M(\alpha)}{M} \ge C_0 Pb \beta^2\alpha$, and so $\epsilon^{\mathit{sample}}(\alpha) = C_0 Pb\beta^2\alpha$ is a lower bound on the per-sample privacy loss. In contrast, in a comparable HFL computation with PBM~\cite{chen2022poisson}, the per-sample privacy loss is bounded by $C_0\frac{Pb\beta^2\alpha}{B}$, where $B$ is the number of samples in a distributed mean computation. This is expected --- vertical distribution of data causes the algorithm to reveal a noisy sum for each sample, while horizontal distribution reveals a noisy sum of $B$ gradients of individual samples.
\end{remark}

\subsection{Convergence}

We next present our theoretical result on the convergence of Algorithm~\ref{vflmpc.alg}. The proof is provided in the appendix.
\begin{theorem} \label{conv.thm}
Under Assumptions~\ref{smooth.assum}-\ref{embedgrad.assum}, if $\eta < \frac{1}{2L}$, then
the average squared gradient over $T$ iterations of Algorithm~\ref{vflmpc.alg} satisfies:
\begin{align}
& \frac{1}{T} \sum_{t=0}^{T-1} \| \nabla \Loss(\Thetab^t) \|^2  \leq \nonumber \\ 
&~~~~~~~~\frac{2 (  \Loss(\Thetab^0) - \mathbb{E}_{t} (\Loss (\Thetab^{T}))}{\eta T} \nonumber + 2 L \eta \frac{ \sigma^2}{B} \nonumber \\
&~+ (1 + 2 L \eta) \left(\frac{C^2 M P(L_0^2 + L_{\hat{h}}^2 \sum_{m=1}^M H_m^2)}{4\beta^2 b}\right). \label{convthm.eq}
\end{align}
\end{theorem}

The first term in the bound in (\ref{convthm.eq}) is determined by the difference between the initial loss and the loss after $T$ training iterations. This term vanishes as $T$ goes to infinity. The second term is the convergence error associated with variance of the stochastic gradients and the Lipschitz constant $L$. The third term is the convergence error arises from the DP noise in the sums of the embeddings. This error depends on the inverse of $b$ and the inverse square of $\beta$, which controls the degree of privacy. As $b$ or $\beta$ increases, this error decreases.

\begin{remark}[Asymptotic convergence]
If $\eta = \frac{1}{\sqrt{T}}$ and $B$ is independent of $T$ then  
\begin{align}
& \frac{1}{T} \sum_{t=0}^{T-1} \| \nabla \Loss(\Thetab^t) \|^2 = O(\sqrt{T} + \mathcal{E})
\end{align}
where $\mathcal{E} = O(\frac{1}{\beta^2  b})$ is the error due to the PBM. 
\end{remark}

We note that if the embedding sums are computed exactly, giving up privacy, the algorithm reduces to standard SGD; the third term in (\ref{convthm.eq}) becomes 0, giving a convergence rate of $O(\frac{1}{\sqrt{T}})$.

\subsection{Tradeoffs}\label{tradeoff}
We observe that there is a connection between the algorithm privacy (both feature and sample), communication cost, and convergence behavior. Let us consider a fixed value for the privacy parameter $\beta$. We can reduce the convergence error  in Theorem~\ref{conv.thm} by increasing $b$, but this results in less privacy guarantee. Higher $b$ also enlarges the algorithm communication cost, but so long as $\log b < F$, this increase is negligible.

Similarly, if we fix the communication cost of the algorithm over $T$ iterations, we can increase the privacy by decreasing $\beta$. This, in turn, leads to an increase in the convergence error on the order of $\frac{1}{\beta^2}$.

The number of parties $M$ also affects the privacy budget and convergence error. With larger $M$, each party gets more protection for their data but with larger convergence error.

Since the budget for feature and sample privacy only differ by a factor of $M$ as shown in previous subsections, the tradeoffs above apply to both feature and sample privacy.

\section{Experiments}\label{exp.sec}

\begin{figure*}[ht]
\centering
\begin{subfigure}{0.24\textwidth}
\centering
  \includegraphics[width=\textwidth]{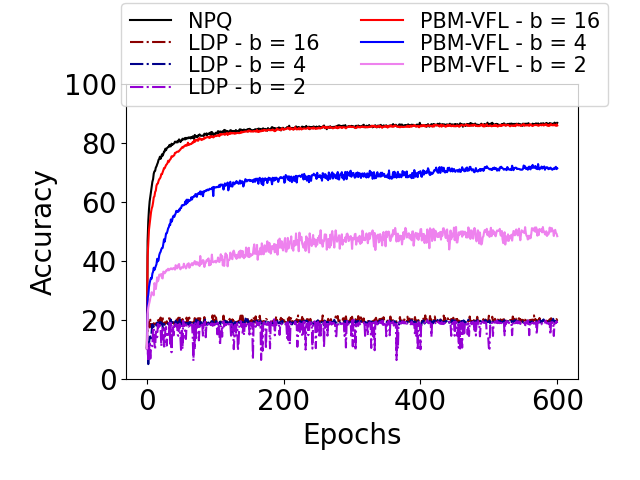}
  \caption{Cifar-10 with $4$ parties\\ and $\beta = 0.1$.}
  \label{resulta.fig}
\end{subfigure}
\begin{subfigure}{0.24\textwidth}
\centering
  \includegraphics[width=\textwidth]{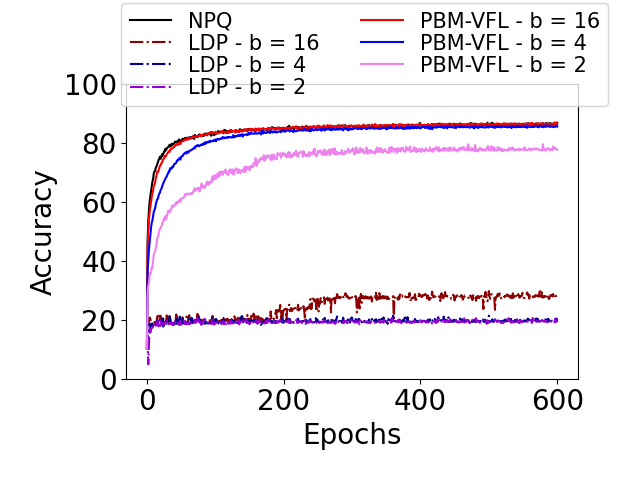}
  \caption{Cifar-10 with $4$ parties\\ and $\beta = 0.15$.}
  \label{resultb.fig}
\end{subfigure}
\begin{subfigure}{0.24\textwidth}
\centering
  \includegraphics[width=\textwidth]{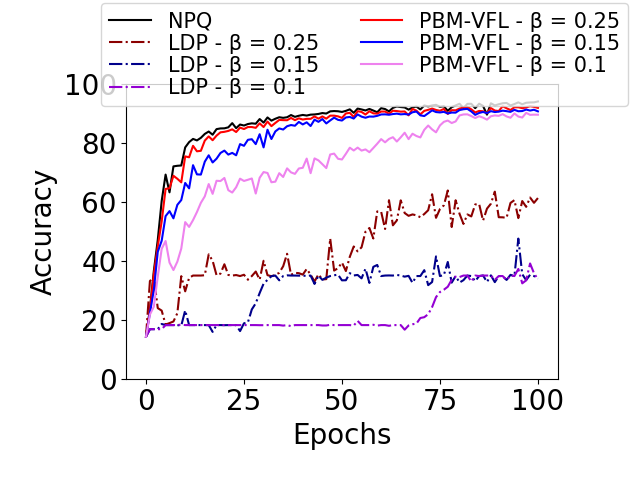}
  \caption{Activity with $5$ parties\\ and $b=64$.}
  \label{resultc.fig}
\end{subfigure}
\begin{subfigure}{0.24\textwidth}
\centering
  \includegraphics[width=\textwidth]{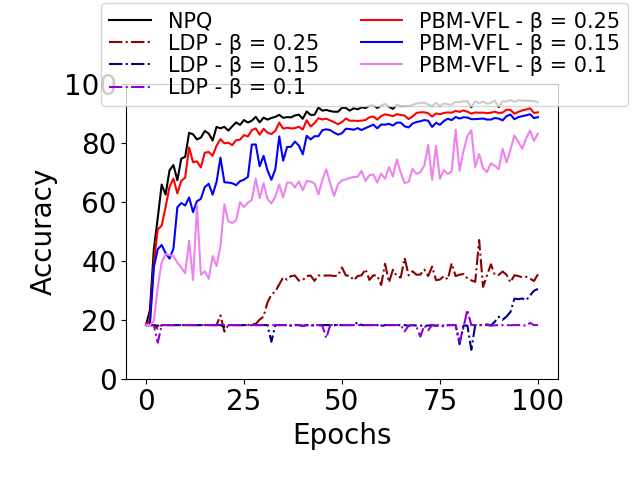}
  \caption{Activity with $10$ parties\\ and $b=64$.}
  \label{resultd.fig}
\end{subfigure}
\caption{Accuracy by epoch on Cifar-10 and Activity. We compare PBM-VFL with No Privacy and Quantization (NPQ) and Local DP (LQP) using Gaussian noise with variance $\sigma_G^2 = \frac{2 M}{b \beta^2}$.}
\label{results.fig}
\end{figure*}

We present experiments to evaluate the tradeoff in accuracy, privacy, and communication cost of PBM-VFL.
In this section, we give results with the following two datasets. The appendix provides more experimental results with these datasets, as well as three 
additional datasets.

\textbf{Activity} \cite{s20082200}:
Time-series positional data of $10,300$ samples with $560$ features for classifying $6$ human activities.
We run experiments with $5$ and $10$ parties, where each party holds $112$ or $56$ features.

\textbf{Cifar-10} \cite{Krizhevsky2009LearningML}:
A dataset of $60,000$ images with $10$ object classes for classification.
We experiment with $4$ parties, each holding a different quadrant of the images.

We use a batch size $B = 100$ and embedding vector size $P = 16$ for both datasets. We train Activity for $100$ epochs and Cifar-10 for $600$ epochs, both with learning rate $0.01$. We consider different sets of PBM parameters: ${b\in \{4,8,16,64,128\}}$ and $\beta \in \{0.1, 0.15, 0.25\}$ for Activity, and ${b\in \{2,4,16\}}$ and $\beta \in \{0.05, 0.1, 0.15\}$ for Cifar-10.

We compare PBM-VFL with two baselines: VFL without privacy and without quantization (NPQ), and VFL with Local DP (LDP). The NPQ method is Algorithm~\ref{vflmpc.alg} without Secure Aggregation or PBM. The LDP method is Algorithm~\ref{vflmpc.alg} without Secure Aggregation, but with local DP provided by adding noise to each embedding before aggregation.
To achieve the same level of feature privacy as PBM, LDP noise is drawn from the Gaussian distribution with variance $\sigma_G^2 = \frac{2 M}{b \beta^2}$. (Since $\epsilon(\alpha)$ is in terms of $b$ and $\beta$, $\sigma$ is in terms of $b$ and $\beta$ as well.)


\textbf{Accuracy and Privacy.} 
Figures~\ref{resulta.fig} and \ref{resultb.fig} show the acccuracy for Cifar-10 for various values of $b$ for $\beta = 0.1$ and $\beta=0.15$.
As discussed in Section \ref{tradeoff}, a higher value of $b$ reduces the convergence error. This is illustrated in both figures: 
the accuracy of PBM-VFL increases as $b$ increases, and $b = 16$ yields almost the same test accuracy as NPQ, while still providing privacy.
In addition, by comparing Figure \ref{resulta.fig} 
and Figure \ref{resultb.fig}, 
we observe that a larger $\beta$ results in better performance for PBM-VFL for all values of $b$.
This makes intuitive sense since larger $\beta$ means that there is less DP noise in the training algorithm. However, this comes with the cost of less privacy.
Notably, PBM-VFL significantly outperforms baseline LDP.
This is expected: since in LDP, each noisy embedding is revealed individually, more noise is required to protect it, and the impact of this noise accumulates over training.

\textbf{Accuracy and Communication Cost.} We summarize the communication cost for PBM-VFL to reach a target accuracy of 80\% on the Cifar-10 dataset in Table \ref{table1}. As described in Section \ref{algorithm.sec}, we compute the total communication cost as $TBMP ( \log(bM) + F)$, with $T$ being the number of iterations needed to reach the training accuracy target and $F = 32$.
Note that for $(b, \beta) = (4, 0.05), (8, 0.05)$, the model does not reach the target accuracy due to high privacy noise levels.
We observe similar trends as in the previous experiments; 
for a given $b$, larger values of $\beta$ reach the accuracy goal in fewer epochs, resulting in lower communication costs. 
Additionally, for a given $\theta$, the number of epochs required to reach the target decreases as we increase $b$. Interestingly,
this results in lower total communication cost, even though a larger $b$  has a higher communication cost per iteration. 
We also observe that with larger $(b,\beta)$, such as $(8,0.15),(16,0.1)$, PBM-VFL significantly reduces communication cost compared to NPQ (no privacy and no quantization) while providing protection to the training data. 

\begin{table}[h]
\caption{Communication cost of PBM-VFL and NPQ on Cifar-10 to reach a train accuracy target of $80\%$.}\label{table1}
\begin{center}
\small
\begin{tabular}{ | c | c | c | c | }
\hline
$b$ & $\beta$ & Number of & Communication \\
& & Epochs & Cost (MB) \\
\hline\hline
& $0.05$ & $\infty$ & $\infty$ \\
$4$ & $0.1$ & $411$ & $4570$ \\ 
& $0.15$ & $54$ & $600$ \\
\hline
& $0.05$ & $\infty$ & $\infty$ \\
$8$ & $0.1$ & $64$ & $725$ \\
& $0.15$ & $38$ & $430$ \\
\hline
& $0.05$ & $442$ & $5110$ \\
$16$ & $0.1$ & $41$ & $470$ \\
& $0.15$ & $30$ & $345$ \\
\hline
\multicolumn{2}{|c|}{NPQ} & $21$ & $4300$ \\
\hline
\end{tabular}
\end{center}
\end{table}

\textbf{Accuracy and Number of Parties.} Figures \ref{resultc.fig} and \ref{resultd.fig} show the  accuracy for different numbers of parties for the Activity dataset, with $b=64$. 
The results show an overall decrease in test accuracy when we increase the number of parties. This is consistent with the theoretical results that convergence error increases as the number of parties grows. We also observe that as in Figures~\ref{resulta.fig} and \ref{resultb.fig}, larger values of $\beta$ results in better performance for PBM-VFL, and that PBM-VFL  outperforms baseline LDP in all cases.

\section{Conclusion}\label{conc.sec}

We presented PBM-VFL, a privacy-preserving and communication-efficient algorithm for training VFL models.
We introduced the novel notion of feature privacy and discussed the privacy budget for feature and sample privacy.
We analyzed privacy and convergence behavior and proved an end-to-end privacy bound as well as a convergence bound.
In future work, we seek to develop appropriate attacks such as data reconstruction and privacy auditing to demonstrate the provided privacy in the VFL model in practice.

\section*{Acknowledgments}
This work was supported by NSF grants CNS-1814898 and CNS-1553340, and by the Rensselaer-IBM AI Research Collaboration (\url{http://airc.rpi.edu}), part of the IBM AI Horizons Network.

\bibliography{citation}

\begin{thebibliography}{29}
\providecommand{\natexlab}[1]{#1}
\providecommand{\url}[1]{\texttt{#1}}
\expandafter\ifx\csname urlstyle\endcsname\relax
  \providecommand{\doi}[1]{doi: #1}\else
  \providecommand{\doi}{doi: \begingroup \urlstyle{rm}\Url}\fi

\bibitem[Agarwal et~al.(2018)Agarwal, Suresh, Yu, Kumar, and Mcmahan]{agarwal2018cpsgd}
Agarwal, N., Suresh, A.~T., Yu, F., Kumar, S., and Mcmahan, H.~B.
\newblock {cpSGD}: Communication-efficient and differentially-private distributed {SGD}, 2018.

\bibitem[Aledhari et~al.(2020)Aledhari, Razzak, Parizi, and Saeed]{9153560}
Aledhari, M., Razzak, R., Parizi, R.~M., and Saeed, F.
\newblock Federated learning: A survey on enabling technologies, protocols, and applications.
\newblock \emph{IEEE Access}, 8:\penalty0 140699--140725, 2020.

\bibitem[Bonawitz et~al.(2016)Bonawitz, Ivanov, Kreuter, Marcedone, McMahan, Patel, Ramage, Segal, and Seth]{bonawitz2016practical}
Bonawitz, K., Ivanov, V., Kreuter, B., Marcedone, A., McMahan, H.~B., Patel, S., Ramage, D., Segal, A., and Seth, K.
\newblock Practical secure aggregation for federated learning on user-held data.
\newblock 2016.

\bibitem[Bottou et~al.(2018)Bottou, Curtis, and Nocedal]{bottou2018optimization}
Bottou, L., Curtis, F.~E., and Nocedal, J.
\newblock Optimization methods for large-scale machine learning.
\newblock \emph{{SIAM} Review}, 60\penalty0 (2):\penalty0 223--311, 2018.

\bibitem[Chen et~al.(2022{\natexlab{a}})Chen, Zhou, Zheng, Wu, Lyu, Wu, Wu, Liu, Wang, and Zheng]{ijcai2022p272}
Chen, C., Zhou, J., Zheng, L., Wu, H., Lyu, L., Wu, J., Wu, B., Liu, Z., Wang, L., and Zheng, X.
\newblock Vertically federated graph neural network for privacy-preserving node classification.
\newblock In \emph{Proc. Thirty-First Int. Joint Conf. Artificial Intelligence}, pp.\  1959--1965, 7 2022{\natexlab{a}}.

\bibitem[Chen et~al.(2022{\natexlab{b}})Chen, {\"{O}}zg{\"{u}}r, and Kairouz]{chen2022poisson}
Chen, W., {\"{O}}zg{\"{u}}r, A., and Kairouz, P.
\newblock The poisson binomial mechanism for unbiased federated learning with secure aggregation.
\newblock In \emph{Proc. Int. Conf. Machine Learning}, pp.\  3490--3506, 2022{\natexlab{b}}.

\bibitem[Deng et~al.(2009)Deng, Dong, Socher, Li, Li, and Fei-Fei]{5206848}
Deng, J., Dong, W., Socher, R., Li, L.-J., Li, K., and Fei-Fei, L.
\newblock Imagenet: A large-scale hierarchical image database.
\newblock In \emph{2009 IEEE Conference on Computer Vision and Pattern Recognition}, pp.\  248--255, 2009.

\bibitem[Dwork \& Roth(2014)Dwork and Roth]{10.1561/0400000042}
Dwork, C. and Roth, A.
\newblock The algorithmic foundations of differential privacy.
\newblock \emph{Found. Trends Theor. Comput. Sci.}, 9\penalty0 (3–4):\penalty0 211–407, Aug 2014.

\bibitem[Garcia-Gonzalez et~al.(2020)Garcia-Gonzalez, Rivero, Fernandez-Blanco, and Luaces]{s20082200}
Garcia-Gonzalez, D., Rivero, D., Fernandez-Blanco, E., and Luaces, M.~R.
\newblock A public domain dataset for real-life human activity recognition using smartphone sensors.
\newblock \emph{Sensors}, 20\penalty0 (8), 2020.
\newblock ISSN 1424-8220.

\bibitem[Geiping et~al.(2020)Geiping, Bauermeister, Dr{\"{o}}ge, and Moeller]{DBLP:conf/nips/GeipingBD020}
Geiping, J., Bauermeister, H., Dr{\"{o}}ge, H., and Moeller, M.
\newblock Inverting gradients - how easy is it to break privacy in federated learning?
\newblock \emph{Adv. Neural Inf. Process. Syst.}, 2020.

\bibitem[Gu et~al.(2021)Gu, Xu, Huo, Deng, and Huang]{gu2021privacy}
Gu, B., Xu, A., Huo, Z., Deng, C., and Huang, H.
\newblock Privacy-preserving asynchronous vertical federated learning algorithms for multiparty collaborative learning.
\newblock \emph{{IEEE} Trans. on Neural Netw. Learn. Syst.}, pp.\  1--13, 2021.

\bibitem[Guo et~al.(2023)Guo, Chaudhuri, Stock, and Rabbat]{guo2023privacyaware}
Guo, C., Chaudhuri, K., Stock, P., and Rabbat, M.
\newblock Privacy-aware compression for federated learning through numerical mechanism design, 2023.

\bibitem[Hu et~al.(2019)Hu, Niu, Yang, and Zhou]{FDML}
Hu, Y., Niu, D., Yang, J., and Zhou, S.
\newblock {FDML:} {A} collaborative machine learning framework for distributed features.
\newblock \emph{Proc. {ACM} Int. Conf. Knowl. Discov. Data Min.}, pp.\  2232--2240, 2019.

\bibitem[Krizhevsky(2009)]{Krizhevsky2009LearningML}
Krizhevsky, A.
\newblock Learning multiple layers of features from tiny images.
\newblock Technical report, 2009.

\bibitem[Li et~al.(2023{\natexlab{a}})Li, Yao, and Liu]{10.5555/3618408.3619245}
Li, S., Yao, D., and Liu, J.
\newblock Fedvs: straggler-resilient and privacy-preserving vertical federated learning for split models.
\newblock ICML'23. JMLR.org, 2023{\natexlab{a}}.

\bibitem[Li et~al.(2023{\natexlab{b}})Li, Wang, and Li]{10.14778/3583140.3583146}
Li, Z., Wang, T., and Li, N.
\newblock Differentially private vertical federated clustering.
\newblock \emph{Proc. VLDB Endow.}, 16\penalty0 (6):\penalty0 1277–1290, Apr 2023{\natexlab{b}}.

\bibitem[Lu \& Ding(2020)Lu and Ding]{9343209}
Lu, L. and Ding, N.
\newblock Multi-party private set intersection in vertical federated learning.
\newblock In \emph{Proc. IEEE 19th Int. Conf. Trust, Security and Privacy in Computing and Communications}, pp.\  707--714, 2020.

\bibitem[Mahendran \& Vedaldi(2015)Mahendran and Vedaldi]{MahendranV15}
Mahendran, A. and Vedaldi, A.
\newblock Understanding deep image representations by inverting them.
\newblock \emph{Proc. {IEEE} Int. Conf. Comput. Vis.}, pp.\  5188--5196, 2015.

\bibitem[McMahan et~al.(2017)McMahan, Moore, Ramage, Hampson, and y~Arcas]{pmlr-v54-mcmahan17a}
McMahan, B., Moore, E., Ramage, D., Hampson, S., and y~Arcas, B.~A.
\newblock Communication-efficient learning of deep networks from decentralized data.
\newblock \emph{Proc. 20th Int. Conf. on Artif. Intell.}, pp.\  1273--1282, 2017.

\bibitem[Mironov(2017)]{Mironov_2017}
Mironov, I.
\newblock R{\'{e}}nyi differential privacy.
\newblock In \emph{Proc. {IEEE} 30th Computer Security Foundations Symp.}, 2017.

\bibitem[Mohammad \& McCluskey(2015)Mohammad and McCluskey]{misc_phishing_websites_327}
Mohammad, R. and McCluskey, L.
\newblock {Phishing Websites}.
\newblock UCI Machine Learning Repository, 2015.

\bibitem[Nguyen et~al.(2018)Nguyen, Nguyen, van Dijk, Richt{\'{a}}rik, Scheinberg, and Tak{\'{a}}c]{HogWild18}
Nguyen, L.~M., Nguyen, P.~H., van Dijk, M., Richt{\'{a}}rik, P., Scheinberg, K., and Tak{\'{a}}c, M.
\newblock {{SGD}} and {Hogwild!} convergence without the bounded gradients assumption.
\newblock \emph{Proc. Int. Conf. on Machine Learn.}, 80:\penalty0 3747--3755, 2018.

\bibitem[Truex et~al.(2019)Truex, Baracaldo, Anwar, Steinke, Ludwig, Zhang, and Zhou]{10.1145/3338501.3357370}
Truex, S., Baracaldo, N., Anwar, A., Steinke, T., Ludwig, H., Zhang, R., and Zhou, Y.
\newblock A hybrid approach to privacy-preserving federated learning.
\newblock In \emph{Proc. 12th ACM Workshop Artificial Intelligence and Security}, pp.\  1--11, 2019.

\bibitem[Truex et~al.(2020)Truex, Liu, Chow, Gursoy, and Wei]{10.1145/3378679.3394533}
Truex, S., Liu, L., Chow, K.-H., Gursoy, M.~E., and Wei, W.
\newblock Ldp-fed: Federated learning with local differential privacy.
\newblock In \emph{Proc. Third ACM Int. Workshop Edge Systems, Analytics and Networking}, pp.\  61--–66, 2020.

\bibitem[Wei et~al.(2020)Wei, Li, Ding, Ma, Yang, Farokhi, Jin, Quek, and Vincent~Poor]{9069945}
Wei, K., Li, J., Ding, M., Ma, C., Yang, H.~H., Farokhi, F., Jin, S., Quek, T. Q.~S., and Vincent~Poor, H.
\newblock Federated learning with differential privacy: Algorithms and performance analysis.
\newblock \emph{IEEE Trans. Inf. Forensics Secur.}, 15:\penalty0 3454--3469, 2020.

\bibitem[Wu et~al.(2015)Wu, Song, Khosla, Yu, Zhang, Tang, and Xiao]{wu20153d}
Wu, Z., Song, S., Khosla, A., Yu, F., Zhang, L., Tang, X., and Xiao, J.
\newblock {3D} shapenets: A deep representation for volumetric shapes.
\newblock \emph{Proc. {IEEE} Int. Conf. Comput. Vis.}, pp.\  1912--1920, 2015.

\bibitem[Xu et~al.(2021)Xu, Baracaldo, Zhou, Anwar, Joshi, and Ludwig]{xu2021fedv}
Xu, R., Baracaldo, N., Zhou, Y., Anwar, A., Joshi, J., and Ludwig, H.
\newblock Fedv: Privacy-preserving federated learning over vertically partitioned data.
\newblock In \emph{Proc. 14th ACM Workshop Artificial Intelligence and Security}, pp.\  181--192, 2021.

\bibitem[Yang et~al.(2019)Yang, Liu, Chen, and Tong]{DBLP:journals/tist/YangLCT19}
Yang, Q., Liu, Y., Chen, T., and Tong, Y.
\newblock Federated machine learning: Concept and applications.
\newblock \emph{{ACM} Trans. Intell. Syst. Technol.}, 10\penalty0 (2):\penalty0 12:1--12:19, 2019.

\bibitem[Youn et~al.(2023)Youn, Hu, Ziani, and Abernethy]{youn2023randomized}
Youn, Y., Hu, Z., Ziani, J., and Abernethy, J.
\newblock Randomized quantization is all you need for differential privacy in federated learning.
\newblock In \emph{Federated Learning and Analytics in Practice: Algorithms, Systems, Applications, and Opportunities}, 2023.

\end{thebibliography}
\bibliographystyle{icml2024}

\newpage
\appendix
\onecolumn

\section{Proofs}

\subsection{Proof of Theorem \ref{privacy.thm}} \label{pravacy_proofs.app}
We prove the theorem by induction over the terms of the sum with $k$ ranging from 1 to $M-1$. The base case is $k=1$ and $\mathcal{P}=\mathcal{P}_{h_1+...+h_{M-1}+h_M}$, $\mathcal{Q}=\mathcal{P}_{h_1+...+h'_{M-1}+h'_M}$, and $\mathcal{R}=\mathcal{P}_{h_1+...+h_{M-1}+h'_M}$. Analogously to the reasoning for $M=2$ we presented in the paper, direct application of Corollary 4 yields
\begin{align}
D_\alpha(\mathcal{P}||\mathcal{Q}) &= D_\alpha(\mathcal{P}_{h_1+...+h_{M-1}+h_M} || \mathcal{P}_{h_1+...+h'_{M-1}+h'_M}) \\
&\le \bigg(2^2\alpha + \frac{1}{\alpha-1}\bigg)\frac{Pb\beta^2}{M}
\nonumber
\end{align}
which fits into the general form in the theorem.

For the inductive step, let $\mathcal{P}=\mathcal{P}_{h_1...+h_{M-k+1}+...+h_M}$, $\mathcal{Q}=\mathcal{P}_{h_1...+h'_{M-k+1}+...+h'_M}$, and let $\mathcal{R}=\mathcal{P}_{h_1...+h_{M-k+1}+h_{M-k}...+h'_M}$ and assume 
\begin{align}
D_\alpha(\mathcal{P}_{h_1...+h_{M-k+1}+...+h_M} || \mathcal{P}_{h_1...+h'_{M-k+1}+...+h'_M}) \le
\bigg((2^{k+1}-2^{k-1}-2)\alpha - T_k + \frac{2^{k-1}-1}{2^{k-2}(\alpha-1)}\bigg)\frac{Pb\beta^2}{M} \label{ind_hypothesis}
\end{align}
where $T_k = 3\cdot2^{k-1} - 3k$.

We need to show 
\begin{align}
D_\alpha(\mathcal{P}_{h_1...+h_{M-k}+...+h_M} || \mathcal{P}_{h_1...+h'_{M-k}+...+h'_M}) &=
D_\alpha(\mathcal{P}_{h_1...+h'_{M-k}+...+h'_M} || \mathcal{P}_{h_1...+h_{M-k}+...+h_M}) \\
&\le \bigg((2^{k+2}-2^{k}-2)\alpha - T_{k+1} + \frac{2^{k}-1}{2^{k-1}(\alpha-1)}\bigg)\frac{Pb\beta^2}{M}
\label{theorem_statement}
\end{align}
where $T_{k+1} = 3\cdot2^{k} + 3(k+1)$.

We will apply Corollary 4 again. Let $\mathcal{P}'=\mathcal{P}_{h_1...+h'_{M-k}+...+h'_M}$, 
$\mathcal{Q}' = \mathcal{P}_{h_1...+h_{M-k}+...+h_M}$, and let $\mathcal{R}' = \mathcal{P}_{h_1...+h_{M-k}+h'_{M-k+1}...+h'_M}$. By Corollary 4
\begin{align}
D_\alpha(\mathcal{P}'||\mathcal{Q}') \le 
&\frac{2\alpha-1}{2(\alpha-1)}D_{2\alpha}(\mathcal{P}'||\mathcal{R}') + D_{2\alpha-1}(\mathcal{R}'||\mathcal{Q}') .
\label{cor4} 
\end{align}
We have
\begin{align}
D_{2\alpha}(\mathcal{P}'||\mathcal{R}') \le \frac{b\beta^2 2\alpha}{M} \label{two_alpha_term}
\end{align}
as $\mathcal{P}'$ and $\mathcal{R}'$ differ in a single embedding, $h'_{M-k}$ and $h_{M-k}$. By the inductive hypothesis (\ref{ind_hypothesis}) we have
\begin{align}
D_\alpha(\mathcal{R}'||\mathcal{Q}') \le 
\bigg((2^{k+1}-2^{k-1}-2)\alpha - T_k + \frac{2^{k-1}-1}{2^{k-2}(\alpha-1)}\bigg)\frac{Pb\beta^2}{M}
\end{align}
and plugging in $2\alpha-1$ for $\alpha$ yields
\begin{align}
D_{2\alpha-1}(\mathcal{R}'||\mathcal{Q}') \le 
\bigg((2^{k+1}-2^{k-1}-2)(2\alpha-1) - T_k + \frac{2^{k-1}-1}{2^{k-1}(\alpha-1)}\bigg)\frac{Pb\beta^2}{M} . \label{two_alpha_minus_one_term}
\end{align}
Substituting (\ref{two_alpha_term}) and (\ref{two_alpha_minus_one_term}) into (\ref{cor4}), then grouping by terms for $\alpha$, a scalar term $T$ and a term for $\frac{1}{\alpha-1}$, yields
\begin{align}
D_\alpha(\mathcal{P}'||\mathcal{Q}') \le 
\bigg((2^{k+2}-2^k-2)\alpha - T_{k+1} + \frac{2^{k}-1}{2^{k-1}(\alpha-1)}\bigg)\frac{Pb\beta^2}{M}
\end{align}
where $T_{k+1} = 2^{k+1}-2^{k-1}-2-1+T_k = 3\cdot2^{k-1} - 3 + T_k$. Since $T_k = 3\cdot2^{k-1}-3k$ (by (\ref{ind_hypothesis})), it follows immediately that $T_{k+1} = 3\cdot2^{k}-3(k+1)$. This is precisely the $S_{k+1}(\alpha)$ term in (\ref{theorem_statement}).

\subsection{Proof of Theorem \ref{conv.thm}} \label{proofs.app}

Let $\Thetab^t$ be the set of model parameters in iteration $t$. 
For brevity, we let $\hat{h}_i^t$ denote $\hat{h}(\thetab_1^t, \ldots, \thetab_M^t; \xb^i)$.
We can write the update rule for $\Thetab$ as 
\begin{align}
\Thetab^{t+1} &= \Thetab^t - \eta G^t 
\end{align}
with
\begin{align}
G^t := \frac{1}{B} \sum_{i \in \B^t} \nabla \ell(\thetab_0^t, \hat{h}_i^t + \varepsilon_i^t)
\end{align}
where $\varepsilon_i^t$ is the $P$-vector of noise resulting from the PBM for the embedding sum of sample $i$ in iteration $t$.

With some abuse of notation, we let $\hat{h}^t$ denote concatenation of the embedding sums  for the minibatch $\mathcal{B}^t$ (without noise) and let $\nabla \Loss_{\B^t}(\thetab_0^t, \hat{h}^t)$ denote the average stochastic gradient of $\Loss$ over minibatch $\B^t$.

We first bound the difference between $G^t$ and $\nabla \Loss_{\B^t}(\thetab_0^t, \hat{h}^t)$ in the following lemma.
\begin{lemma} \label{gradvar.lem}
It holds that
\begin{align}
\mathbb{E}_{\B^t} \| G^t - \nabla \Loss_{\B^t}(\thetab_0^t, \hat{h}^t) \|^2 \leq
\frac{C^2 M P(L_0^2 + L_{\hat{h}}^2 \sum_{m=1}^M H_m^2)}{4 \beta^2 b}
\end{align}
\end{lemma}

\begin{proof}
We first note that 
\begin{align}
\mathbb{E}_{\B^t} \| G^t - \nabla \Loss_{\B^t}(\thetab_0^t, \hat{h}^t) \|^2 = \sum_{m=0}^M \mathbb{E}_{\B^t} \| G_m^t - \nabla_m \Loss_{\B^t}(\thetab_0^t, \hat{h}^t)  \|^2 \label{sumdif.eq}
\end{align}
where $G_m^t$ is the block of $G^t$ corresponding to party $m$.

For $m=0$, using Assumption~\ref{smooth.assum}, we can bound the first term in the summation in (\ref{sumdif.eq}) as 
\begin{align}
\mathbb{E}_{\B^t} \| G_0^t - \nabla_0 \Loss_{\B^t}(\thetab_0, \hat{h})  \|^2 &\leq \frac{1}{B} \sum_{i \in \B^t} \mathbb{E}_{\B^t} \| \nabla \ell(\thetab_0^t, \hat{h}_i^t + \varepsilon_i^t) - \nabla \ell(\thetab_0^t, \hat{h}_i^t) \|^2 \\
&\leq \frac{L_0^2}{B} \sum_{i \in \B^t} \mathbb{E}_{\B^t} \| \varepsilon_i^t \|^2. \label{serverbound.eq}
\end{align}

For $m=1, \ldots M$, by the chain rule, we have
\begin{align}
&G_m^t = \frac{1}{B} \sum_{i \in \mathcal{B}} \nabla_m h_m(\thetab_m^t; \xb_i) \nabla_{h_m} \hat{h}_i^t  \nabla_{\hat{h}} \ell(\thetab_0^t, \hat{h}_i^t + \varepsilon_i^t) \\
&\nabla_m \Loss_{\B^t}(\theta_0^t, \hat{h}^t) = 
\frac{1}{B} \sum_{i \in \mathcal{B}} \nabla_m h_m(\thetab_m^t; \xb_i) \nabla_{h_m} \hat{h}_i^t  \nabla_{\hat{h}} \ell(\thetab_0^t, \hat{h}_i^t). 
\end{align}
It follows that
\begin{align}
\mathbb{E}_{\B^t} \| G_m^t - \nabla_m \Loss_{\B}(\thetab_0^t, \hat{h}^t) \|^2 =
\mathbb{E}_{\B^t} \left( \frac{1}{B^2} \sum_{i \in \mathcal{B}^t} \| \nabla_m h_m(\thetab_m^t; \xb_i) \nabla_{h_m} \hat{h}_i^t   \right.
\left(\nabla_{\hat{h}} \ell(\thetab_0^t, \hat{h}_i^t + \varepsilon_i^t) - \nabla_{\hat{h}} \ell(\thetab_0^t, \hat{h}_i^t\right) ) \|^2 \Bigg).
\end{align}
Noting that $ \nabla_{\hat{h}} \hat{h}_i^t = \bm{I}$, we have
\begin{align}
\mathbb{E}_{\B^t} \| G_m^t - \nabla_m \Loss_{\B}(\thetab_0^t, \hat{h}^t) \|^2
&\leq \frac{1}{B} \sum_{i \in \mathcal{B}^t} \mathbb{E}_{\B^t} \|\nabla_m h_m(\thetab_m^t; \xb_i)\|_{\mathcal{F}}^2 \| \nabla_{\hat{h}} \ell(\thetab_0^t, \hat{h}_i^t + \varepsilon_i^t)
- \nabla_{\hat{h}} \ell(\thetab_0^t, \hat{h}_i^t) \|^2  \label{bounds1.eq} \\
&\leq  \frac{H_m^2 L_{\hat{h}}^2}{B} \sum_{i \in \mathcal{B}^t} \mathbb{E}_{\B^t} \|\varepsilon_i^t \|^2 \label{bounds2.eq}
\end{align}
where (\ref{bounds2.eq}) follows from (\ref{bounds1.eq}) by Assumptions~\ref{smooth.assum} and \ref{embedgrad.assum}.

We combine (\ref{serverbound.eq}) and (\ref{bounds2.eq}) to obtain
\begin{align}
\mathbb{E}_{\B^t} \| G^t - \nabla \Loss_{\B^t}(\thetab_0, \hat{h}^t) \|^2 &\leq
\left(L_0^2 + L_{\hat{h}}^2 \sum_{m=1}^M H_m^2\right) \frac{1}{B} \sum_{i \in \mathcal{B}^t} \mathbb{E}_{\B^t} \|\varepsilon_i^t \|^2  \label{prechen.eq} \\
&\leq \frac{C^2 M P(L_0^2 + L_{\hat{h}}^2 \sum_{m=1}^M H_m^2)}{4 \beta^2 b} \label{postchen.eq}
\end{align}
where (\ref{postchen.eq}) follows from (\ref{prechen.eq}) by Theorem~\ref{pbm.thm}.
\end{proof}

We now prove the main theorem.
\begin{proof}
By Assumption~\ref{smooth.assum}, we have
\begin{align}
\Loss (\Thetab^{t+1}) - \Loss(\Thetab^t)
&\leq -\langle \nabla \Loss(\Thetab^t), \Thetab^{t+1} - \Thetab^t \rangle + \frac{L}{2}\| \Thetab^{t+1} - \Thetab^t  \|^2 \\
&= -\eta \langle  \nabla \Loss(\Thetab^t), G^t \rangle + \frac{L \eta^2}{2} \| G^t \|^2 \\
&= -\eta \langle  \nabla \Loss(\Thetab^t), G^t - \nabla \Loss(\Theta^t) \rangle
- \eta \langle \nabla \Loss(\Theta^t), \nabla \Loss_{\B^t}(\Theta^t) \rangle \nonumber \\
&~~~ + \frac{L \eta^2}{2} \| G^t - \nabla \Loss_{\B^t}(\Theta^t) + \nabla \Loss_{\B^t}(\Theta^t)  \|^2 \\
&\leq -\eta \langle  \nabla \Loss(\Thetab^t), G^t - \nabla \Loss(\Theta^t) \rangle - \eta \langle \nabla \Loss(\Theta^t), \nabla \Loss_{\B^t}(\Theta^t) \rangle \nonumber \\
&~~~ + L \eta^2 \| G^t - \nabla \Loss_{\B^t}(\Theta^t) \|^2  +  L \eta^2 \| \nabla \Loss_{\B^t}(\Theta^t)  \|^2. 
\end{align}
Taking expectation with respect to $t$, conditioned on $\Theta^t$:
\begin{align}
\mathbb{E}_{t} \left(\Loss (\Thetab^{t+1})\right) - \Loss(\Thetab^t) &\leq \frac{\eta}{2} \| \nabla \Loss(\Theta^t) \|^2 + \frac{\eta}{2} \mathbb{E}_t \| G^t - \nabla \Loss_{\B^t}(\Theta^t) \|^2 - \eta \langle \nabla \Loss(\Theta^t) , \mathbb{E}_t \left( \nabla \Loss_{\B^t}(\Theta^t) \right) \rangle \nonumber \\
&~~~ + L \eta^2 \mathbb{E}_t \| G^t - \nabla \Loss_{\B^t}(\Theta^t) \|^2  +  L \eta^2 \mathbb{E}_t \| \nabla \Loss_{\B^t}(\Theta^t)  \|^2 \label{expec1.eq} \\
&= - \frac{\eta}{2} \| \nabla \Loss(\Theta^t) \|^2 + \frac{\eta}{2}\left( 1 + 2L\eta \right) \mathbb{E}_t \| G^t - \nabla \Loss_{\B^t}(\Theta^t) \|^2  \nonumber \\
&~~~ + L \eta^2 \mathbb{E}_t \| \nabla \Loss_{\B^t}(\Theta^t)  \|^2 \label{expec2.eq} \\
&= - \frac{\eta}{2} \| \nabla \Loss(\Theta^t) \|^2
+ \frac{\eta}{2}\left( 1 + 2L\eta \right) \mathbb{E}_t \| G^t - \nabla \Loss_{\B^t}(\Theta^t) \|^2  \nonumber \\
&~~~ + L \eta^2 \mathbb{E}_t \| \nabla \Loss_{\B^t}(\Theta^t) - \nabla \Loss(\Thetab^t)  \|^2 
+ L \eta^2 \mathbb{E}_t \| \nabla \Loss(\Theta^t)  \|^2 \label{expec3.eq} \\
&\leq - \frac{\eta}{2}\left( 1 - 2L \eta \right) \| \nabla \Loss(\Theta^t) \|^2 
+ \frac{\eta}{2}\left( 1 + 2L\eta \right) \mathbb{E}_t \| G^t - \nabla \Loss_{\B^t}(\Theta^t) \|^2  \nonumber \\
&~~~ + L \eta^2 \mathbb{E}_t \| \nabla \Loss_{\B^t}(\Theta^t) - \nabla \Loss(\Thetab^t)  \|^2 \label{expec4.eq}
\end{align}
where (\ref{expec2.eq}) follows from (\ref{expec1.eq}) by Assumption~\ref{bias.assum}, and (\ref{expec3.eq}) follows from (\ref{expec2.eq}) because $A \cdot B  = \frac{1}{2} A^2 + \frac{1}{2} B^2 - \frac{1}{2}(A - B)^2.$

Applying Assumption~\ref{var.assum} and Lemma~\ref{gradvar.lem} to (\ref{expec4.eq}), we obtain
\begin{align}
\mathbb{E}_{t} \left(\Loss (\Thetab^{t+1})\right) - \Loss(\Thetab^t) \nonumber &\leq - \frac{\eta}{2}\left( 1 - 2L \eta \right) \| \nabla \Loss(\Theta^t) \|^2
+  \frac{\eta}{2}(1 + 2 L \eta) \left(\frac{C^2 M P(L_0^2 + L_{\hat{h}^2} \sum_{m=1}^M H_m^2)}{4\beta^2 b}\right) 
+ \frac{L \eta^2 \sigma^2}{B}. \label{expec5.eq}
\end{align}

Applying the assumption that $\eta < \frac{1}{2L}$ and rearranging (\ref{expec5.eq}), we obtain
\begin{align}
\| \nabla \Loss(\Thetab^t) \|^2 &\leq \frac{2(  \Loss(\Thetab^t) - \mathbb{E}_{t} (\Loss (\Thetab^{t+1}))} {\eta} 
+ (1 + 2 L \eta) \left(\frac{C^2 M P(L_0^2 + L_{\hat{h}}^2 \sum_{m=1}^M H_m^2)}{4\beta^2 b}\right) + 2 L \eta \frac{\sigma^2}{B}.
\end{align}

Averaging over $T$ iterations and taking total expectation, we have 
\begin{align}
\frac{1}{T} \sum_{t=0}^{T-1} \| \nabla \Loss(\Thetab^t) \|^2  \leq \frac{2 (  \Loss(\Thetab^0) - \mathbb{E}_{t} (\Loss (\Thetab^{T}))}{\eta T}
+ (1 + 2 L \eta) \left(\frac{C^2 M P(L_0^2 + L_{\hat{h}}^2 \sum_{m=1}^M H_m^2)}{4\beta^2 b}\right)
+ 2 L \eta \frac{ \sigma^2}{B}.
\end{align}
\end{proof}

\section{Experimental Details}

We describe the each dataset and its experimental setup below.

\textbf{Activity} is a time-series positional data on humans performing six activities, including walking, walking upstairs, walking downstairs, sitting, standing, and laying. The dataset is released under the CC0: Public Domain license, and is available for download on \url{https://www.kaggle.com/datasets/uciml/human-activity-recognition-with-smartphones/data}. The dataset contains $10,300$ samples, including $7,353$ training samples and $2,948$ testing samples. 
We run experiments with $5$ and $10$ parties, where each party holds $112$ or $56$ features.
We use batch size $B = 100$ and embedding size $P = 16$, and train for $100$ epochs with learning rate $0.01$.
We consider different sets of PBM parameters ${b\in \{4,8,16,64,128\}}$ and $\beta \in \{0.1, 0.15, 0.2, 0.25\}$.
The experimental results are computed using the average of $3$ runs.

\textbf{Cifar-10} is an image dataset with $10$ object classes for classification task. The dataset is released under the MIT License. Details about Cifar-10 can be found at \url{https://www.cs.toronto.edu/~kriz/cifar.html}. The dataset can be downloaded via \verb|torchvision.datasets.CIFAR10|. The dataset consists of $60,000$ $32$x$32$ colour images in $10$ classes, with $6,000$ images per class. The dataset is divided into training set of $50,000$ images and test set of $10,000$ images.
We experiment with $4$ parties, each holding a different quadrant of the images.
We use batch size $B = 100$ and embedding size $P = 16$, and train for $600$ epochs with learning rate $0.01$.
We consider different sets of PBM parameters ${b\in \{2,4,8,16\}}$ and $\beta \in \{0.05, 0.1, 0.15\}$.
The experimental results are computed using the average of $3$ runs.

\textbf{ImageNet} \cite{5206848} is a large-scale image dataset used for classification. The dataset is released under the CC-BY-NC 4.0 license. Details about ImageNet can be found at \url{https://www.image-net.org/}. We use a random $100$-class subset from the 2012 ILSVRC version of the data, including $100,000$ training images and $26,000$ testing images.
We run experiments with $4$ parties, where each party holds a different quadrant of the images.
We use batch size $B = 256$ and embedding size $P = 128$, and train for $700$ epochs with learning rate $0.03$.
We consider different sets of PBM parameters $b \in \{2, 4, 8\}$ and $\beta \in \{0.01, 0.05, 0.1\}$.
We show the results of a single run due to the large size of the dataset.


\textbf{ModelNet-10} \cite{wu20153d} is a large collection of 3D CAD models of different objects taken in $12$ different views.
The dataset is released under the MIT License.
Details about the ModelNet-10 dataset are available at \url{https://modelnet.cs.princeton.edu/}, and we used this \href{https://drive.google.com/file/d/0B4v2jR3WsindMUE3N2xiLVpyLW8/view}{Google Drive link} to download the dataset.
We train our model with a subset of the dataset with $1,008$ training samples and $918$ test samples on the set of $10$ classes: bathtub, bed, chair, desk, dresser, monitor, night stand, sofa, table, toilet.
We run experiments with $6$ and $12$ parties, where each party holds $2$ or $1$ view(s) of each CAD model.
We use batch size $B = 64$ and embedding size $P = 4096$, and train for $250$ epochs with learning rate $0.01$.
We consider different sets of PBM parameters $b \in \{2, 4, 8, 16\}$ and $\beta \in \{0.05, 0.1, 0.15\}$.
The experimental results are computed using the average of $3$ runs.

\textbf{Phishing} \cite{misc_phishing_websites_327} is a tabular dataset of $11,055$ samples with $30$ features for classifying if a website is a phishing website. The features include information about the use of HTTP, TinyURL, forwarding, etc. The dataset is released under the CC-BY-NC 4.0 license, and is available for download at \url{https://www.openml.org/search?type=data&sort=runs&id=4534&status=active}.
We split the dataset into training set and testing set with ratio $0.8$.
We experiment with $5$ and $10$ parties, where each party holds $6$ or $3$ features.
We use batch size $B = 100$ and embedding size $P = 16$, and train for $100$ epochs with learning rate $0.01$. We consider different sets of PBM parameters $b \in \{8, 16, 32, 64\}$ and $\beta \in \{0.1, 0.15, 0.2, 0.25\}$.
The experimental results are computed using the average of $3$ runs.

We implemented our experiment using a computer cluster of $40$ nodes. Each node is a CentOS 7 with $2$x$20$-core $2.5$ GHz Intel Xeon Gold $6248$ CPUs, 8× NVIDIA Tesla V100 GPUs with $32$ GB HBM, and $768$ GB of RAM. We provide our complete code and running instructions as part of the supplementary material.

The model neural network architecture for each of the five dataset are described as follow. Each party model of ModelNet-10 is a neural network with two convolutional layers and a fully-connected layer. Phishing and Activity each has a 3-layer dense neural network as the party model. Cifar-10 uses a ResNet18 neural network for each party model, and ImageNet uses a ResNet18 neural network for each party model. To bound the embedding values into the range $[-C, C]$ as required for Algorithm~\ref{pbm.alg}, we use the $tanh$ activation function to scale the embedding values with $C=1$ for each party model of all datasets.
All five datasets have the same server model that consists of a fully-connected layer for classification with cross-entropy loss.

\section{Additional Experimental Results}

In this section, we include additional results from the experiments introduced in Section \ref{exp.sec}, as well as experiments with the additional datasets.

\subsection{Accuracy and Privacy}

\begin{figure*}[h!]
\centering
\begin{subfigure}{0.33\textwidth}
\centering
  \includegraphics[width=\textwidth]{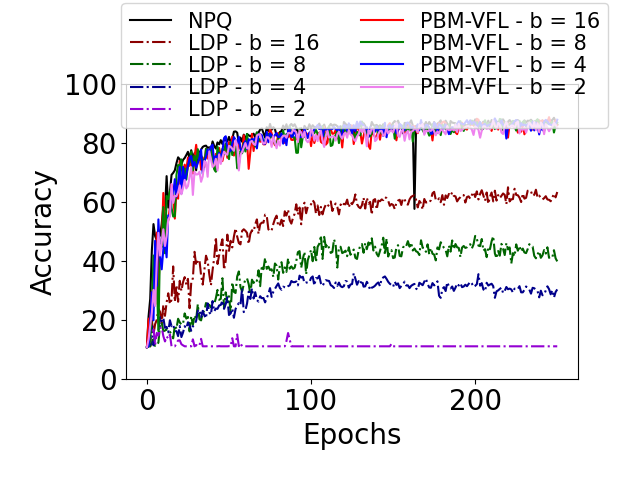}
  \caption{$\beta = 0.05$}
  \label{fig3a}
\end{subfigure}
\begin{subfigure}{0.33\textwidth}
  \centering
  \includegraphics[width=\textwidth]{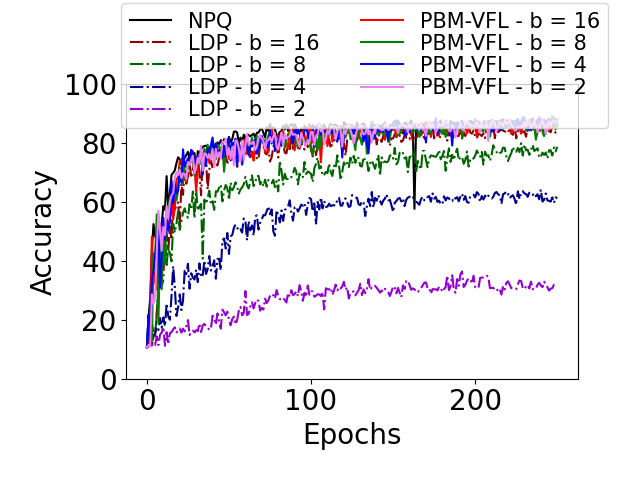}
  \caption{$\beta = 0.1$}
  \label{fig3b}
\end{subfigure}
\begin{subfigure}{0.33\textwidth}
\centering
  \includegraphics[width=\textwidth]{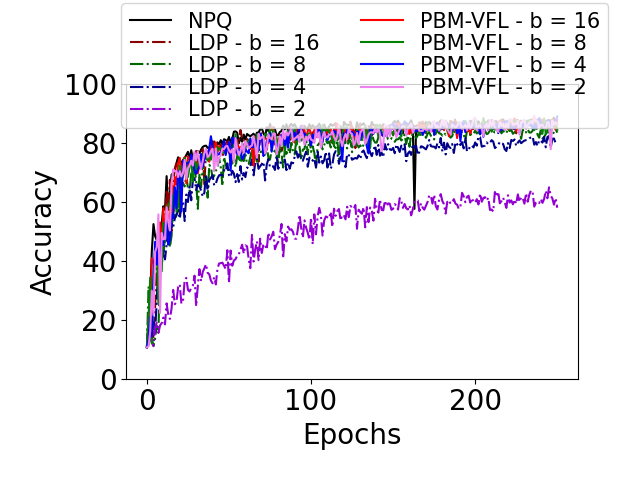}
  \caption{$\beta = 0.15$}
  \label{fig3c}
\end{subfigure}
\caption{Test accuracy by epoch on ModelNet-10 with $6$ parties. We compare PBM-VFL with No Privacy and Quantization (NPQ) and Local DP (LQP) using Gaussian noise with variance $\sigma_G^2 = \frac{2 M}{b \beta^2}$.}
\label{fig3}
\end{figure*}

Figure \ref{fig3} plots the test accuracy of ModelNet-10 dataset with $6$ parties. In Figures~\ref{fig3a}, \ref{fig3b}, and \ref{fig3c}, we see that the model performance of PBM-VFL is nearly the same as the baseline algorithm implementation without any privacy and quantization. Moreover, PBM-VFL outperforms the Local DP baseline in all cases. While PBM-VFL and Local DP provide the same level of privacy, given the same values of $b$ and $\beta$,  PBM-VFL achieves better accuracy.

\begin{figure*}[h!]
\centering
\begin{subfigure}{0.33\textwidth}
\centering
  \includegraphics[width=\textwidth]{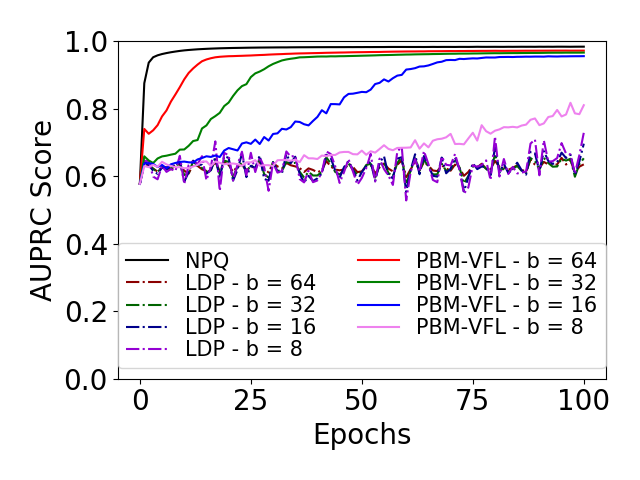}
  \caption{$\beta = 0.1$}
  \label{fig4a}
\end{subfigure}
\begin{subfigure}{0.33\textwidth}
  \centering
  \includegraphics[width=\textwidth]{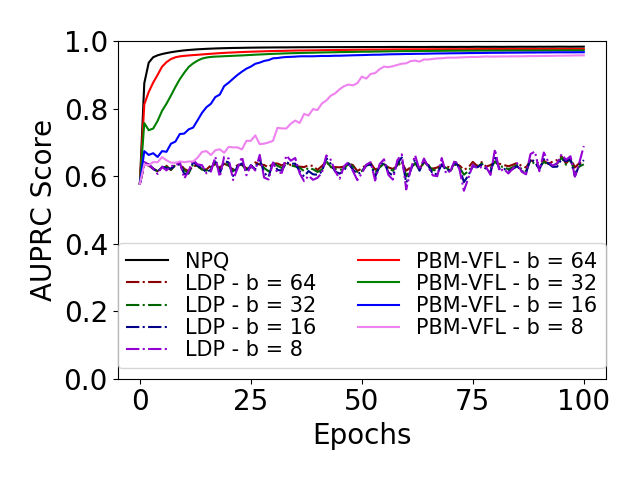}
  \caption{$\beta = 0.15$}
  \label{fig4b}
\end{subfigure}
\begin{subfigure}{0.33\textwidth}
\centering
  \includegraphics[width=\textwidth]{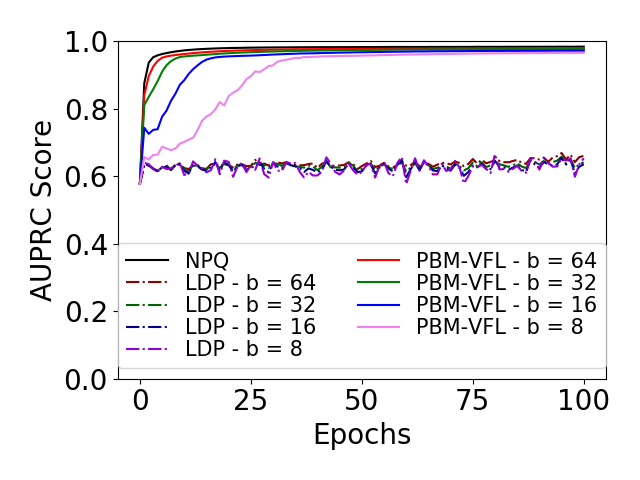}
  \caption{$\beta = 0.2$}
  \label{fig4c}
\end{subfigure}
\caption{AUPRC score by epoch on Phishing with $10$ parties. We compare PBM-VFL with No Privacy and Quantization (NPQ) and Local DP (LQP) using Gaussian noise with variance $\sigma_G^2 = \frac{2 M}{b \beta^2}$.}
\label{fig4}
\end{figure*}

Figure \ref{fig4} shows the AUPRC score of Phishing dataset with $10$ parties. For every fixed value of $\beta$, the model performance improves as $b$ increases. We also get accuracy improvement by fixing $b$ and increase $\beta$ as shown across Figures~\ref{fig4a}, ~\ref{fig4b}, and~\ref{fig4c}. This trend matches our theoretical results, as well as experimental results in Section~\ref{exp.sec}. In addition, our algorithm PBM-VFL outperforms the baselines with Local DP.

\begin{figure*}[h!]
\centering
\begin{subfigure}{0.3\textwidth}
\centering
  \includegraphics[width=\textwidth]{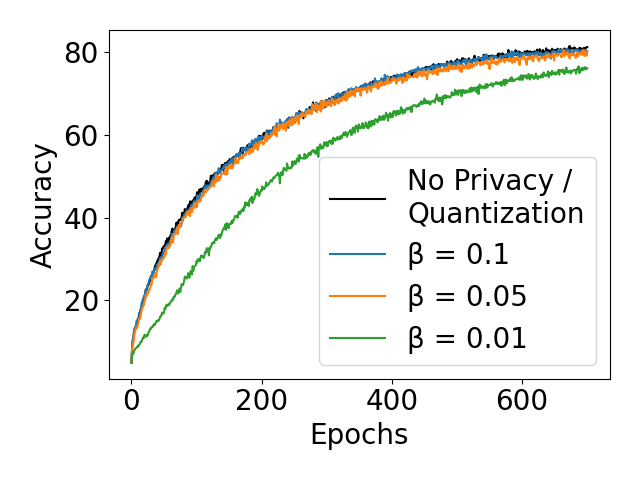}
  \caption{$b = 2$}
  \label{imageneta}
\end{subfigure}
\begin{subfigure}{0.3\textwidth}
  \centering
  \includegraphics[width=\textwidth]{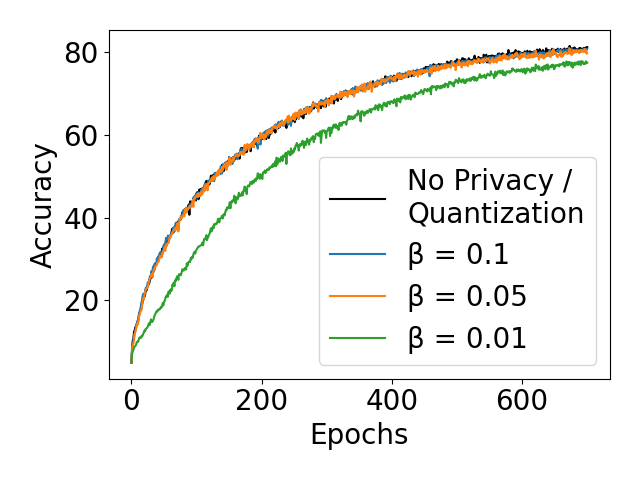}
  \caption{$b = 4$}
  \label{imagenetb}
\end{subfigure}
\begin{subfigure}{0.3\textwidth}
\centering
  \includegraphics[width=\textwidth]{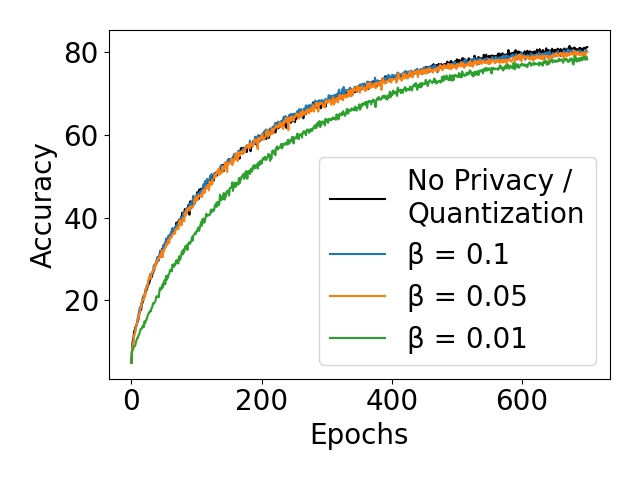}
  \caption{$b = 8$}
  \label{imagenetc}
\end{subfigure}
\caption{Test accuracy by epoch on ImageNet with $4$ parties. Test accuracy with $\beta = 0.05$ and $\beta = 0.1$ is nearly the same with the base case without Secure Aggregation and DP. Test accuracy with fixed $\beta = 0.01$ (green lines) increases as $b$ grows.}
\label{imagenet}
\end{figure*}

Figure \ref{imagenet} shows the test accuracy of PBM-VFL on the large-scale ImageNet dataset, and we observe similar trend as described above. With higher values of $\beta = 0.05$ and $\beta = 0.1$, there is nearly no loss in the test accuracy compared to the base case without any privacy and quantization. Even with very small value of $\beta = 0.01$ which provides high level of privacy, the test accuracy is still very high.

\subsection{Accuracy and Communication Cost}

\begin{table}[h!]
\begin{center}
\captionof{table}{Activity 5 clients - Train Accuracy Target $80\%$} \label{table4}
\begin{tabular}{ | c | c | c | c | }
\hline
$b$ & $\beta$ & Number of & Communication \\
& & Epochs & Cost (GB) \\
\hline\hline
& $0.1$ & $\infty$ & $\infty$ \\
$8$ & $0.15$ & $\infty$ & $\infty$ \\
& $0.2$ & $92$ & $1.94$ \\
& $0.25$ & $66$ & $1.39$ \\
\hline
& $0.1$ & $\infty$ & $\infty$ \\
$16$ & $0.15$ & $92$ & $1.98$ \\
& $0.2$ & $54$ & $1.16$ \\
& $0.25$ & $31$ & $0.67$ \\
\hline
& $0.1$ & $56$ & $1.25$ \\
$64$ & $0.15$ & $23$ & $0.51$ \\
& $0.2$ & $12$ & $0.27$ \\
& $0.25$ & $10$ & $0.22$ \\
\hline
& $0.1$ & $24$ & $0.55$ \\
$128$ & $0.15$ & $12$ & $0.27$ \\
& $0.2$ & $10$ & $0.22$ \\
& $0.25$ & $10$ & $0.22$ \\
\hline
\multicolumn{2}{|c|}{NPQ} & $10$ & $0.38$ \\
\hline
\end{tabular}
\end{center}
\end{table}

\begin{table}[h!]
\begin{center}
\captionof{table}{Phishing 5 clients - Train AUPRC Score Target $0.9$} \label{table6}
\begin{tabular}{ | c | c | c | c | }
\hline
$b$ & $\beta$ & Number of & Communication \\
& & Epochs & Cost (GB) \\
\hline\hline
& $0.1$ & $\infty$ & $\infty$ \\
$8$ & $0.15$ & $86$ & $2.18$ \\
& $0.2$ & $41$ & $1.04$ \\
& $0.25$ & $23$ & $0.58$ \\
\hline
& $0.1$ & $98$ & $2.54$ \\
$16$ & $0.15$ & $34$ & $0.88$ \\
& $0.2$ & $15$ & $0.39$ \\
& $0.25$ & $8$ & $0.21$ \\
\hline
& $0.1$ & $35$ & $0.92$ \\
$32$ & $0.15$ & $12$ & $0.32$ \\
& $0.2$ & $5$ & $0.13$ \\
& $0.25$ & $3$ & $0.08$ \\
\hline
& $0.1$ & $15$ & $0.40$ \\
$64$ & $0.15$ & $4$ & $0.11$ \\
& $0.2$ & $3$ & $0.08$ \\
& $0.25$ & $2$ & $0.05$ \\
\hline
\multicolumn{2}{|c|}{NPQ} & $2$ & $0.09$ \\
\hline
\end{tabular}
\end{center}
\end{table}

Tables \ref{table4} and \ref{table6} show the total communication cost needed to reach $80\%$ training accuracy for Activity and $0.9$ AUPRC training score for phishing. We observe that the communication cost decreases as we increase $b$. For the same value of $b$, higher $\beta$ also reduce the total communication cost. This behavior matches the trend described in Section \ref{exp.sec}.

\newpage

\subsection{Accuracy and Number of Parties}

\begin{figure*}[ht]
\centering
\begin{subfigure}{0.45\textwidth}
\centering
  \includegraphics[width=\textwidth]{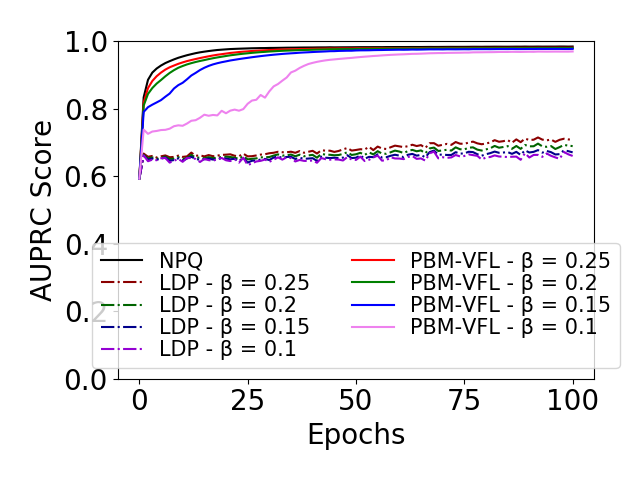}
  \caption{$5$ parties and $b = 32$}
  \label{phishinga}
\end{subfigure}
\begin{subfigure}{0.45\textwidth}
  \includegraphics[width=\textwidth]{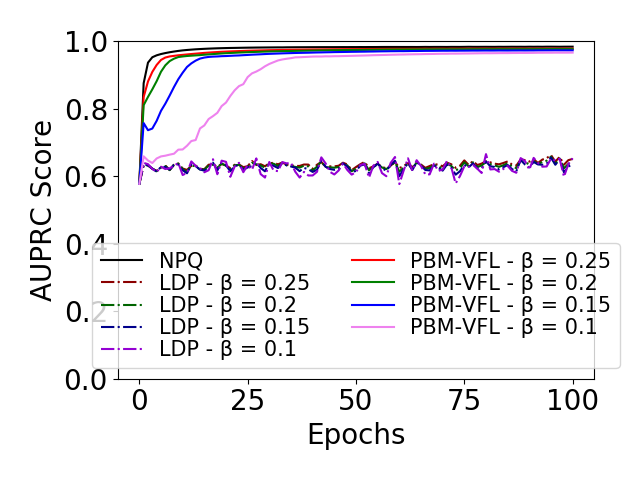}
  \caption{$10$ parties and $b = 32$}
  \label{phishingb}
\end{subfigure}
\caption{AUPRC score by epoch on Phishing dataset. We compare PBM-VFL with No Privacy and Quantization (NPQ) and Local DP (LQP) using Gaussian noise with variance $\sigma_G^2 = \frac{2 M}{b \beta^2}$.}
\label{phishing}
\end{figure*}

In Figures \ref{phishing} and \ref{modelnet}, we observe the same increasing trend in test accuracy for a smaller number of parties as described in Section~\ref{exp.sec}. For a fixed value of $b$, or fixed communication cost, a larger number of parties results in higher convergence error, leading to lower test accuracy across all $\beta$ values.

\begin{figure*}[ht]
\centering
\begin{subfigure}{0.45\textwidth}
\centering
  \includegraphics[width=\textwidth]{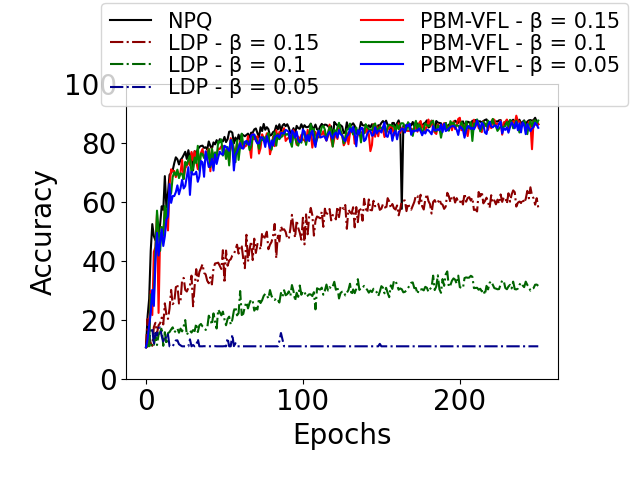}
  \caption{$6$ parties and $b = 2$}
  \label{modelneta}
\end{subfigure}
\begin{subfigure}{0.45\textwidth}
  \includegraphics[width=\textwidth]{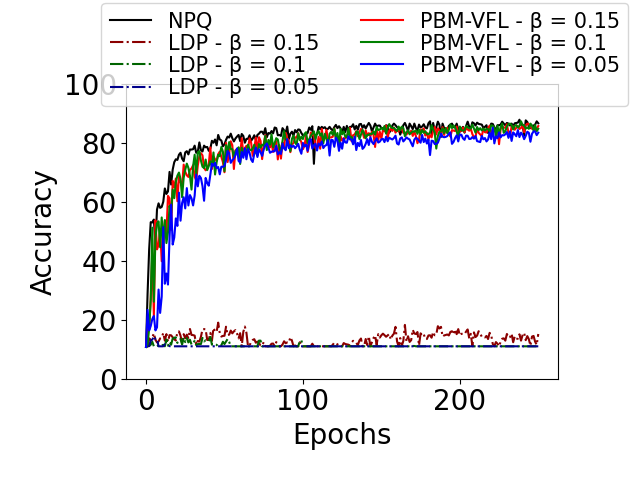}
  \caption{$12$ parties and $b = 2$}
  \label{modelnetb}
\end{subfigure}
\caption{Test accuracy by epoch on ModelNet-10 dataset. We compare PBM-VFL with No Privacy and Quantization (NPQ) and Local DP (LQP). LDP is local Gaussian noise with variance $\sigma_G^2 = \frac{2 M}{b \beta^2}$.}
\label{modelnet}
\end{figure*}

\end{document}